\documentclass[11pt,letterpaper]{article}

 \usepackage{booktabs}       

\usepackage{tcs}

\usepackage{mysymbols} 

\newcommand{\err}{\text{err}}
\newcommand{\indic}[1]{ \mathbb{I} \tlprn {#1}}

\title{Robust Linear Regression: Optimal Rates in Polynomial Time}

\author{ Ainesh Bakshi\thanks{AB would like to thank the partial support
from the Office of Naval Research (ONR) grant N00014-18-1-2562, and the National Science Foundation (NSF) under
Grant No. CCF-1815840.}\\
abakshi@cs.cmu.edu\\
CMU
\and
Adarsh Prasad\\
adarshp@andrew.cmu.edu\\
CMU
  }
\date{}
\begin{document}

\maketitle

\begin{abstract}
We obtain robust and computationally efficient estimators for learning several linear models that achieve statistically optimal convergence rate under minimal distributional assumptions. Concretely, we assume our data is drawn from a $k$-hypercontractive distribution and an $\epsilon$-fraction is adversarially corrupted. We then describe an estimator that converges to the optimal least-squares minimizer for the true distribution at a rate proportional to $\epsilon^{2-2/k}$, when the noise is independent of the covariates. We note that no such estimator was known prior to our work, even with access to unbounded computation. The rate we achieve is information-theoretically optimal and thus we
resolve the main open question in Klivans, Kothari and Meka [COLT'18].

Our key insight is to identify an analytic condition 
that serves as a polynomial relaxation of independence of random variables. In particular, we show that when the moments of the noise and covariates are negatively-correlated, we obtain the same rate as independent noise. Further, when the condition is not satisfied, we obtain a rate proportional to $\epsilon^{2-4/k}$, and again match the information-theoretic lower bound. Our central technical contribution is to algorithmically exploit independence of random variables in the "sum-of-squares" framework by formulating it as the aforementioned polynomial inequality.

\end{abstract}

\thispagestyle{empty}

\clearpage


\microtypesetup{protrusion=false}
\tableofcontents{}
\thispagestyle{empty}
\microtypesetup{protrusion=true}

\clearpage

\pagestyle{plain}
\setcounter{page}{1}

\section{Introduction}

While classical statistical theory has focused on designing statistical estimators assuming access to i.i.d. samples from a nice distribution, estimation in the presence of adversarial outliers has been a challenging problem since it was formalized by Huber \cite{huber1964robust}. A long and influential line of work in high-dimensional robust statistics has since focused on studying the trade-off between sample complexity, accuracy and more recently, computational complexity for basic tasks such as estimating mean, covariance ~\cite{lai2016agnostic,diakonikolas2016robust,DBLP:conf/stoc/CharikarSV17,DBLP:journals/corr/abs-1711-11581,DBLP:journals/corr/SteinhardtCV17,DBLP:conf/soda/0002D019,DBLP:conf/icml/DiakonikolasKK017,DBLP:conf/soda/DiakonikolasKK018,DBLP:conf/colt/0002D0W19}, moment tensors of distributions ~\cite{DBLP:journals/corr/abs-1711-11581} and regression~\cite{DBLP:conf/focs/DiakonikolasKS17,DBLP:conf/colt/KlivansKM18,DBLP:conf/icml/DiakonikolasKK019,prasad2018robust,karmalkar2019list,raghavendra2020regression}.

Regression continues to be extensively studied under various models, including realizable regression (no noise), true linear models (independent noise), asymmetric noise, agnostic regression and generalized linear models~(see~\cite{weisberg2005applied} and references therein). In each model, a variety of distributional assumptions are considered over the covariates and the noise. As a consequence,
there exist innumerable estimators for regression achieving various trade-offs between sample complexity, running time and rate of convergence. The presence of adversarial outliers adds yet another dimension to design and compare estimators. 

Seminal works on robust regression focused on designing non-convex loss functions, including M-estimators \cite{huber2011robust}, Theil-Sen estimators\cite{theil1992rank, sen1968estimates},  R-estimators\cite{jaeckel1972estimating}, Least-Median-Squares \cite{rousseeuw1984least} and S-estimators\cite{rousseeuw1984robust}. These estimators have desirable statistical properties under disparate assumptions, yet remain computationally intractable in high dimensions.   
Further, recent works show that it is information-theoretically impossible to design robust estimators for linear regression without distributional assumptions \cite{DBLP:conf/colt/KlivansKM18}. 

An influential recent line of work showed that when the data is drawn from the well studied and highly general class of \textit{hypercontractive} distributions (see Definition \ref{def:hyp}), there exist robust and computationally efficient estimators for regression \cite{DBLP:conf/colt/KlivansKM18, prasad2018robust, conf/soda/DiakonikolasKS19}. 
Several families of natural distributions fall into this category, including Gaussians, strongly log-concave distributions and product distributions on the hypercube. However, both estimators converge to the the true hyperplane (in $\ell_2$-norm) at a sub-optimal rate, as a function of the fraction of corrupted points. 



Given the vast literature on ad-hoc and often incomparable estimators for high-dimensional robust regression, the central question we address in this work is as follows:
\begin{quote}
\centering 
\emph{Does there exist a unified approach to design robust and computationally efficient estimators achieving optimal rates for all linear regression models under mild distributional assumptions? }
\end{quote}

We address the aforementioned question by introducing a framework to design robust estimators for linear regression when the input is drawn from a \textit{hypercontractive} distribution. Our estimators converge to the true hyperplanes at the information-theoretically optimal rate (as a function of the fraction of corrupted data) under various well-studied noise models, including independent and agnostic noise. Further, we show that our estimators can be computed in polynomial time using the \textit{sum-of-squares} convex hierarchy.

We note that, despite decades of progress, prior to our work, estimators achieving optimal convergence rate in terms of the fraction of corrupted points were not known, even with independent noise and access to unbounded computation.

\subsection{Our Results}

We begin by formalizing the regression model we work with. In classical regression, we assume $\calD$ is a distribution over $\R^d \times \R$ and for a vector $\Theta \in \R^d$, the least-squares loss is given by $\err_\calD(\Theta) = \expecf{x,y\sim\calD}{\Paren{ y - x^{\top}\Theta }^2}$. The goal is to learn $\Theta^* = \arg\min_{\Theta} \err_{\calD}(\Theta)$. We assume sample access to $\calD$, and given $n$ i.i.d. samples, we want to obtain a vector $\Theta$ that approximately achieves optimal error, $\err_\calD(\Theta^*)$. 

In contrast to the classical setting, we work in the \textit{strong contamination model}. Here, an adversary has access to the input samples and is allowed to corrupt an $\epsilon$-fraction arbitrarily. Note, the adversary has access to unbounded computation and has knowledge of the estimators we design. We note that this is the most stringent corrupt model and captures Huber contamination, additive corruption, label noise, agnostic learning etc (see \cite{diakonikolas2019recent}). Formally, 


\begin{model}[Robust Regression Model]
\label{model:robust_regression}
Let $\calD$ be a distribution over $\R^d \times R$ such that the marginal distribution over $\R^d$ is centered and has covariance $\Sigma^*$ and let $\Theta^* = \arg\min_\Theta \expecf{x,y\sim \calD}{\Paren{ y- \Iprod{\Theta, x} }^2}$ be the optimal hyperplane for $\calD$. Let $\{ (x^*_1, y^*_1), (x^*_2, y^*_2), \ldots (x^*_n, y^*_n)\}$  be $n$ i.i.d. random variables drawn from $\calD$.   Given $\epsilon>0$, the robust regression model $\calR_{\calD}(\epsilon, \Sigma^*, \Theta^*)$  outputs a set of $n$ samples $\{ (x_1, y_1), \ldots (x_n, y_n) \}$ such that for  at least $(1-\epsilon)n$ points $x_i = x^*_i$ and $y_i = y_i^*$. The remaining $\epsilon n$ points are arbitrary, and potentially adversarial w.r.t. the input and estimator.
\end{model}

A natural starting point is to assume that the marginal distribution over the covariates (the $x$'s above) is heavy-tailed and has bounded, finite covariance. However, we show that there is no robust estimator in this setting, even when the linear model has no noise and the uncorrupted points lie on a line.

\begin{theorem}[Bounded Covariance does not suffice, Theorem \ref{thm:lb_bounded_covariance} informal]
For all $\epsilon >0$, there exist two distributions $\calD_1, \calD_2$ over $\R^d \times \R$ such that the marginal distribution over the covariates has bounded covariance, denoted by $\Sigma^2 = \Theta(1)$,  yet $\Norm{ \Sigma^{1/2}\Paren{\Theta_1 - \Theta_2} }_2 = \Omega(1)$, where $\Theta_1$ and $\Theta_2$ are the optimal hyperplanes for $\calD_1$ and $\calD_2$. 
\end{theorem}

The aforementioned result precludes any statistical estimator that converges to the true hyperplane as the fraction of corrupted points tends to $0$. Therefore, we strengthen the distributional assumption and hypercontractive distributions instead.  
\begin{definition}[$(C,k)$-Hypercontractivity]
\label{def:hyp}
A distribution $\calD$ over $\mathbb{R}^d$ is $(C,k)$-hypercontractive for an even integer $k\geq 4$, if for all $r \in [k/2]$, for all $v \in \mathbb{R}^d$,
\[ 
\expecf{x \sim \calD }{ \Iprod{v, x - \expecf{}{x}}^{2r}} \leq \expecf{x \sim \calD }{ C \Iprod{v,x - \expecf{}{x} }^2}^{r} 
\]
\end{definition}
\begin{remark}
Hypercontractivity captures a broad class of distributions, including Gaussian distributions, uniform distributions over the hypercube and sphere, affine transformations of isotropic distributions satisfying Poincare inequalities \cite{KothariSteinhardt17} and strongly log-concave distributions. Further, hypercontractivity is preserved under natural closure properties like affine transformations, products and weighted mixtures \cite{kothari2017outlier}. Further, efficiently computable estimators appearing in this work require \textit{certifiable}-hypercontractivity (Definition \ref{def:certifiable_hypercontractivity}), a strengthening that continues to capture aforementioned distribution classes.
\end{remark}



In this work we focus on the \textit{rate of convergence} of our estimators to the true hyperplane, $\Theta^*$, as a function of the fraction of corrupted points, denoted by $\epsilon$. We measure convergence in both parameter distance ($\ell_2$-distance between the hyperplanes) and least-squares error on the true distribution ($\err_{\calD}$). 

We introduce a simple analytic condition on the relationship between the noise (marginal distribution over $y-x^\top \Theta^*$) and covariates (marginal distribution over $x$) that can  be considered as a proxy for independence of $y-x^\top \Theta^*$ and $x$ :

\begin{definition}[Negatively Correlated Moments]
\label{def:negative_moments}
Given a distribution $\calD$ over $\R^d\times \R$, such that the marginal distribution on $\R^d$ is $(c_k,k)$-hypercontractive, the corresponding regression instance has negatively correlated moments if for all $r \leq k$, and for all $v$, 
\[
\expecf{x,y \sim \calD }{\Iprod{v,x}^r \Paren{y - x^{\top}\Theta^*}^r} \leq \bigO{1} \expecf{x \sim \calD }{\Iprod{v,x}^r} \expecf{x,y \sim \calD}{\Paren{y - x^{\top}\Theta^*}^r}
\]
\end{definition}

Informally, the \textit{negatively correlated moments} condition can be viewed as a polynomial relaxation of independence of random variables. Note, it is easy to see that when the noise  is independent of the covariates, the above definition is satisfied.

\begin{remark}
We show that when this condition is satisfied by the true distribution,  $\calD$, we obtain rates that match the information theoretically optimal rate in a \textit{true linear model}, where the noise (marginal distribution over $y - x^\top \Theta^*$) is independent of the covariates (marginal distribution over $x$). Further, when this condition is not satisfied, we show that there exist distributions for which obtaining rates matching the \textit{true linear model} is impossible. 
\end{remark}

When the distribution over the input is hypercontractive and has negatively correlated moments, we obtain an estimator achieving \textit{rate} proportional to $\epsilon^{1-1/k}$ for parameter recovery. Further, our estimator can be computed efficiently. Thus, our main algorithmic result is as follows:  

\begin{theorem}[Robust Regresssion with Negatively Correlated Noise, Theorem \ref{thm:optimal_efficient_robust_regression} informal]
\label{thm:informal_negative_moment_polytime}
Given $\epsilon>0, k \geq 4$, and $n \geq \Paren{d\log(d)}^{\mathcal{O}(k)}$ samples from $\mathcal{R}_{\calD}(\epsilon,  \Sigma^*, \Theta^*)$, such that $\calD$ is $(c,k)$-certifiably hypercontractive and has negatively correlated moments, there exists an algorithm that runs in $n^{\mathcal{O}(k)}$ time and outputs an estimator $\tilde\Theta$ such that with high  probability,
\[
    \Norm{ (\Sigma^*)^{1/2}\Paren{\Theta^* - \tilde\Theta} }_2  \leq \bigO{\epsilon^{1-1/k}} \Paren{ \err_{\calD}(\Theta^*)^{1/2}  }
\]
and, 
\[
\err_{\calD}(\tilde\Theta) \leq \Paren{1 + \bigO{ \epsilon^{2-2/k}}} \err_{\calD}(\Theta^*) 
\]
\end{theorem}

\begin{remark}
We note that prior work does not draw a distinction between the independent and dependent noise models. 
In comparison (see Table \ref{tab:sumarry_of_res}), Klivans, Kothari and Meka \cite{DBLP:conf/colt/KlivansKM18} obtained a sub-optimal least-squares error scales proportional to $\epsilon^{1-2/k}$. For the special case of $k=4$, Prasad et. al.~\cite{prasad2018robust} obtain least squares error proportional to $O(\epsilon \kappa^2(\Sigma))$, where $\kappa$ is the condition number. In very recent independent work Zhu, Jiao and Steinhardt~\cite{zhu2020robust} obtained a sub-optimal least-squares error scales proportional to $\epsilon^{2-4/k}$. 
\end{remark}


Further, we show that the rate we obtained in Theorem \ref{thm:informal_negative_moment_polytime}
is information-theoretically optimal, even when the noise and covariates are independent: 

\begin{theorem}[Lower Bound for Independent Noise,  Theorem \ref{thm:lower_bound_independent_noise} informal ]
\label{thm:informal_independent_noise_lowerbound}
For any $\epsilon >0$, there exist two distributions $\calD_1, \calD_2$ over $\R^2 \times\R $ such that the marginal distribution over $\R^2$ has covariance $\Sigma$ and is $(c,k)$-hypercontractive for both distributions, and  yet $\Norm{\Sigma^{1/2}(\Theta_{1} - \Theta_{2} ) }_2 =\bigOmega{ \epsilon^{1-1/k} \sigma } $, where $\Theta_1 , \Theta_2$ are the optimal hyperplanes for $\calD_1 $ and $\calD_2$ respectively, $\sigma = \max( \err_{\calD_1}(\Theta_1),  \err_{\calD_2}(\Theta_2))$ and the noise is uniform over $[-\sigma, \sigma]$. Further, $\abs{\err_{\calD_1}(\Theta_2) - \err_{\calD_1}(\Theta_1)} = \bigOmega{ \epsilon^{2-2/k} \sigma^2 } $.
\end{theorem}

Next, we consider the setting where the noise is allowed to arbitrary, and need not have negatively correlated moments with the covariates. A simple modification to our algorithm and analysis yields an efficient estimator that obtains rate proportional to $\epsilon^{1-2/k}$ for parameter recovery. 

\begin{corollary}[Robust Regresssion with Dependent Noise, Corollary \ref{cor:dependent_noise} informal]
\label{cor:informal_arbitrary_noise_polytime}
Given $\epsilon>0, k\geq 4$ and $n\geq\Paren{d\log(d)}^{\mathcal{O}(k)}$ samples from $\mathcal{R}_{\calD}(\epsilon,  \Sigma^*, \Theta^*)$, such that $\calD$ is $(c,k)$-certifiably hypercontractive, there exists an algorithm that runs in $n^{\mathcal{O}(k)}$ time and outputs an estimator $\tilde\Theta$ such that with probability 9/10,
\[
    \Norm{ (\Sigma^*)^{1/2}\Paren{\Theta^* - \tilde\Theta} }_2  \leq \bigO{\epsilon^{1-2/k}} \Paren{ \err_{\calD}(\Theta^*)^{1/2}  }
\]
and, 
\[
\err_{\calD}(\tilde\Theta) \leq \Paren{1 + \bigO{ \epsilon^{2-4/k}}} \err_{\calD}(\Theta^*) 
\]
\end{corollary}

Further, we show that the dependence on $\epsilon$ is again information-theoretically optimal: 

\begin{theorem}[Lower Bound for Dependent Noise, Theorem \ref{thm:lower_bound_dependent_noise} informal]
\label{thm:informal_dependent_noise_lowerbound}
For any $\epsilon >0$, there exist two distributions $\calD_1, \calD_2$ over $\R^2 \times\R $ such that the marginal distribution over $\R^2$ has covariance $\Sigma$ and is $(c,k)$-hypercontractive for both distributions, and  yet $\Norm{\Sigma^{1/2}(\Theta_{1} - \Theta_{2} ) }_2 =\bigOmega{ \epsilon^{1-2/k} \sigma}$, where $\Theta_1 , \Theta_2$ be the optimal hyperplanes for $\calD_1 $ and $\calD_2$ respectively and $\sigma = \max( \err_{\calD_1}(\Theta_1),  \err_{\calD_2}(\Theta_2))$. Further, $\abs{\err_{\calD_1}(\Theta_2) - \err_{\calD_1}(\Theta_1)} = \bigOmega{ \epsilon^{2-4/k} \sigma^2 } $.
\end{theorem}

\begin{table}[t!]
\begin{centering}
\begin{tabular}{|l|c|c|}
\hline
\textbf{Estimator}                                                                                                       & \textbf{Independent Noise}                                                           & \textbf{Arbitrary Noise}                                                             \\ \hline
\begin{tabular}[c]{@{}l@{}}Prasad et. al. \cite{prasad2018robust},\\ Diakonikolas et. al. \cite{diakonikolas2018sever} \end{tabular}                                 & $ \epsilon ~\kappa^2 $ (only $k=4$)  & $ \epsilon ~\kappa^2 $ (only $k=4$)                                                                      \\ \hline
\begin{tabular}[c]{@{}l@{}}Klivans, Kothari and Meka\\ \cite{DBLP:conf/colt/KlivansKM18}\end{tabular}                              & $\epsilon^{1-2/k}$                                                          & $\epsilon^{1-2/k}$                                                                                                                                \\ \hline
\begin{tabular}[c]{@{}l@{}}Zhu, Jiao and Steinhardt\\ \cite{zhu2020robust}\end{tabular}                                  & $\epsilon^{2-4/k}$                                                          & $\epsilon^{2-4/k}$                                                                                                                               \\ \hline
\begin{tabular}[c]{@{}l@{}} \textbf{Our Work}\\ Thm \ref{thm:informal_negative_moment_polytime}, Cor \ref{cor:informal_arbitrary_noise_polytime} \end{tabular}     & $\epsilon^{2-2/k}$                                                          & $\epsilon^{2-4/k}$                                                                                                                                \\ \hline
\begin{tabular}[c]{@{}l@{}}\textbf{Lower Bounds}\\ Thm \ref{thm:informal_independent_noise_lowerbound}, Thm \ref{thm:informal_dependent_noise_lowerbound}\end{tabular} & $\epsilon^{2-2/k}$      & $\epsilon^{2-4/k}$                                                                                  \\ \hline
\end{tabular}
\caption{Comparison of convergence rate (for least-squares error) achieved by various computationally efficient estimators for Robust Regression, when the underlying distribution is $(c_k,k)$-hypercontractive. }
\label{tab:sumarry_of_res}
\end{centering}
\end{table}

\paragraph{Applications for Gaussian Covariates.}
The special case where the marginal distribution over $x$ is Gaussian has received considerable interest recently \cite{diakonikolas2018efficient, diakonikolas2018sever}. We note that our estimators extend to the setting of Gaussian covariates, since the uniform distribution over samples from $\calN(0,\Sigma)$ are $\Paren{\bigO{k}, \bigO{k}}$-certifiably hypercontractive for all $k$ (see Section 5 in Kothari and Steurer \cite{kothari2017outlier}). As a consequence, instantiating Corollary \ref{cor:informal_arbitrary_noise_polytime} with $k = \log(1/\epsilon)$ yields the following:

\begin{corollary}[Robust Regression with Gaussian Covariates]
Given $\epsilon>0$ and $n \geq \Paren{d\log(d)}^{\mathcal{O}(log(1/\epsilon))}$ samples from $\mathcal{R}_{\calN}(\epsilon,  \Sigma^*, \Theta^*)$, such that the marginal distribution over the $x$'s is $\calN(0,\Sigma^*)$, there exists an algorithm that runs in $n^{\mathcal{O}(\log(1/\epsilon)}$ time and outputs an estimator $\tilde\Theta$ such that with high  probability,
\[
    \Norm{ (\Sigma^*)^{1/2}\Paren{\Theta^* - \tilde\Theta} }_2  \leq \bigO{\epsilon \log(1/\epsilon)} \Paren{ \err_{\calN}(\Theta^*) }^{1/2}
\]
and, 
\[
\err_{\calN}(\tilde\Theta) \leq \Paren{1 + \bigO{ \Paren{\epsilon \log(1/\epsilon)}^2}} \err_{\calN}(\Theta^*) 
\]
\end{corollary}

We note that our estimators obtain the rate matching recent work for Gaussians, albeit in quasi-polynomial time. In comparison, Diakonikolas, Kong and Stewart \cite{diakonikolas2018efficient} obtain the same rate in polynomial time, when the noise is independent of the covariates. We note that obtaining the optimal rate  for Gaussian covariates (shaving the additional $\log(1/\epsilon)$ factor) remains an outstanding open question.





\subsection{Related Work.}

\paragraph{Robust Statistics.} Computationally efficient estimators for robust statistics in high dimension have been extensively studied, following the initial work on robust mean estimation \cite{diakonikolas2016robust, lai2016agnostic}. We focus on literature regarding robust regression and sum-of-squares. We refer the reader to recent surveys and theses for an extensive discussion of the literature on robust statistics \cite{raghavendra2018high, li2018principled, steinhardt2018robust, diakonikolas2019recent}.

\paragraph{Robust Regression.}
Computationally efficient estimators for robust linear regression were proposed by~\cite{prasad2018robust,DBLP:conf/colt/KlivansKM18,DBLP:conf/icml/DiakonikolasKK019,conf/soda/DiakonikolasKS19}. While \cite{prasad2018robust}~and~\cite{DBLP:conf/icml/DiakonikolasKK019} obtained estimators for the more general case of distributions with bounded 4th moment. However, their estimators suffer an error of $\bigO{\err_{D}(\Theta^*) \epsilon \kappa^2(\Sigma)}$, where $\kappa(\Sigma)$ is the condition number of the covariance matrix of $X$. Hence, these estimators don't obtain the optimal dependence on $\epsilon$ in the negatively correlated noise setting, and also suffer an additional condition number dependence in the dependent noise setting.~\cite{conf/soda/DiakonikolasKS19} obtained improved bounds under the restrictive assumption that $X$ is distributed according to a Gaussian.~\cite{DBLP:conf/colt/KlivansKM18} obtained polynomial-time estimators for distributions with certifiably bounded distributions, however, their estimators obtain a sub-optimal error of $\bigO{\err_{D}(\Theta^*) \epsilon^{1-2/k}}$. 
In very recent and independent work, Zhu, Jiao and Steinhardt \cite{zhu2020robust} obtained polynomial time estimators for the dependent noise setting, but their estimators are sub-optimal for the negatively correlated setting.

There has been significant work in more restrictive noise models as well. For instance, a series of works \cite{bhatia2015robust, bhatia2017consistent, suggala2019adaptive} consider a noise model where the adversary is only allowed to corrupt the labels and obtain \textit{consistent} estimators in this regime (error goes to zero with more samples). In comparison, our estimators do not obtain  For a comprehensive overview we refer the reader to references in the aforementioned papers.

\paragraph{Sum-of-Squares Algorithms.} 
In recent years, there has been a significant progress in applying the Sum-of-Squares framework to design efficient algorithms for several fundamental computational problems. Starting with the work of Barak, Kelner and Steurer \cite{DBLP:conf/stoc/BarakKS15}, sum-of-squares algorithms have  the best known running time for dictionary learning and tensor decomposition \cite{ DBLP:conf/focs/MaSS16, DBLP:conf/colt/SchrammS17,hopkins2019robust}, optimizing random tensors over the sphere \cite{DBLP:conf/approx/BhattiproluGL17} and refuting CSPs below the spectral threshold \cite{DBLP:conf/stoc/RaghavendraRS17}.

In the context of high-dimensional estimation, sum-of-squares algorithms achieved state-of-the-art performance for robust moment estimation \cite{kothari2017outlier}, robust regression\cite{DBLP:conf/colt/KlivansKM18}, robustly learning mixtures of spherical \cite{KothariSteinhardt17,hopkins2018mixture}  and arbitrary Gaussians \cite{bakshi2020outlier,diakonikolas2020robustly}, heavy-tailed estimation \cite{hopkins2018sub,cherapanamjerialgo2020} and list-decodable variants of these problems \cite{karmalkar2019list, raghavendra2020regression, raghavendra2020list, bakshi2020list, cherapanamjeri2020list}. 

\paragraph{Concurrent Work.} 
We note that a statistical estimator achieving rate proportional to  $\epsilon^{1-1/k}$ can be obtained from combining ideas in \cite{zhu2019generalized} and \cite{zhu2020robust}\footnote{We thank Banghua Zhu, Jiantao Jiao, and Jacob Steinhardt for communicating their observation to us.}. However, this approach remains computationally intractable. 
Finally, Cherapanamjeri et al. \cite{cherapanamjerialgo2020} consider the special case of $k=4$ and obtain nearly linear sample complexity and running time. However, their running time and rate incurs a condition number dependence. Further, their rate scales proportional to $\epsilon^{1/2}$, even when the noise is indpendent of the covariates (as opposed to $\epsilon^{3/4}$).

We emphasize that the bottleneck in all prior and concurrent work remains algorithmically exploiting the independence of the noise and covariates, which we achieve via the \textit{negatively correlated moments} condition (Definition \ref{def:negative_moments}).

\section{Technical Overview}

In this section, we provide an overview of our approach, the new algorithmic ideas we introduce and the corresponding technical challenges. At a high level, we build on several recent works that study Sum-of-Squares relaxations for solving algorithmic problems arising in robust statistics 
Following the \textit{proofs-to-algorithms} paradigm arising from the aforementioned works, we show that given two distributions over regression instances that are close in \textit{total variation distance} (definition \ref{def:tv}), \textit{any} hyperplane minimizing the least-squares loss on one distribution must be close (in $\ell_2$ distance) to any other hyperplane minimizing the loss on the second distribution. 

This information-theoretic statement immediately yields a robust estimator achieving optimal rate, albeit given access to unbounded computation. 
To see this, let $\calD_1$ be the uniform distribution over $n$ i.i.d samples from the true distribution, and $\calD_2$ be the uniform distribution over $n$ corrupted samples from $\calR_{\calD}(\epsilon, \Sigma^*, \Theta^*)$, denoted by $\calD_2$. It is easy to check that the total variation distance between $\calD_1$ and $\calD_2$ is at most $\epsilon$. Therefore, the two hyperplanes must be close in $\ell_2$ norm.
In order to make this strategy algorithmic, we show that we can distilled a set of polynomial constraints from the information theoretic proof and can efficiently optimize over them using the Sum-of-Squares framework. We note that we crucially require several new constraints, including the \textit{gradient} condition and the \textit{negatively-correlated moments} condition  (see Section \ref{para:ncm}), which did not appear in any prior works.
We describe each step in more detail subsequently.

\subsection{Total Variation Closeness implies Hyperplane Closeness}
\label{subsec:tv_to_regressor}
Consider two distributions $\calD_1$ and $\calD_2$ over $\R^d \times \R$ such that the total variation distance between $\calD_1$ and $\calD_2$ is $\epsilon$ and the marginals for both distributions over $\R^d$ are $(c_k, k)$-hypercontractive and have covariance $\Sigma$. Ignoring computational and sample complexity concerns, we can obtain the optimal hyperplanes corresponding to each distribution. Note, these hyperplanes need not be unique and are simply charecterized as minimzers of the least-squares loss : for $i \in \{1,2\}$, 
\[
\Theta_i = \arg\min_{\Theta} \expecf{x,y\sim \calD_i}{ \Paren{y - x^{\top} \Theta }^2}
\]

Our central contribution is to obtain an information theoretic proof that the optimal hyperplanes are indeed close in scaled $\ell_2$ norm, i.e. 
\[
\Norm{\Sigma^{1/2} \Paren{\Theta_1 - \Theta_2}}_2 \leq \bigO{  \epsilon^{1-1/k}} \Paren{\expecf{x,y\sim \calD_1}{ \Paren{y - x^{\top} \Theta_1 }^2}^{1/2} + \expecf{x,y\sim \calD_2}{ \Paren{y - x^{\top} \Theta_2 }^2}^{1/2}} 
\]
We refer the reader to Theorem \ref{thm:robust_identitifiability_independent} for a precise statement. 
Further, we show that the $\epsilon^{1-1/k}$ dependence can be achieved even when the noise is not completely independent of the covariates but satisfies an analytic condition which we refer to as \textit{negatively correlated moments} (see Definition \ref{def:negative_moments}). We provide an outline of the proof as it illustrates where the techniques we introduced in this work. 

\paragraph{Coupling and Decoupling.} 
We begin by considering a maximal coupling, $\calG$, between distributions $\calD_1$ and $\calD_2$ such that they disagree on at most an $\epsilon$-measure support ($\epsilon$-fraction of the points for a discrete distribution). Let $(x,y)\sim\calD_1 $ and $(x',y')\sim \calD_2$. Then, observe for any vector $v$,
\begin{equation}
\label{eqn:1}
    \begin{split}
    \Iprod{v, \Sigma (\Theta_{1} - \Theta_{2})  } &=\Iprod{v, \expecf{\calG}{x x^\top} (\Theta_{1} - \Theta_{2})  } \\
    & =  \expecf{\calG}{\Iprod{v, x \Paren{ x^\top \Theta_{1}  - y}  }} + \expecf{\calG}{\Iprod{v, x \Paren{  y - x^\top \Theta_{2}}   }} 
\end{split}
\end{equation}
While the first term in Equation \eqref{eqn:1} depends completely on $\calD_1$, the second term requires using the properties of the maximal coupling. Since $1 = 1_{(x,y) = (x',y')} + 1_{(x,y) \neq (x',y')}$, we can rewrite the second term in Equation \eqref{eqn:1} as follows: 
\begin{equation}
\label{eqn:2}
\begin{split}
    \expecf{\calG}{\Iprod{v, x \Paren{  y - x^\top \Theta_{2}}   }}  & = \expecf{\calG}{\Iprod{v, x' \Paren{  y' - (x')^\top \Theta_{2}}   } 1_{(x,y) = (x',y')}} \\
    & \hspace{0.2in}+ \expecf{\calG}{\Iprod{v, x \Paren{  y - x^\top \Theta_{2}}   }  1_{(x,y) \neq (x',y')} } 
\end{split}
\end{equation}
With a bit of effort, we can combine Equations \eqref{eqn:1} and \eqref{eqn:2}, and upper bound them as follows (see Theorem \ref{thm:robust_identitifiability_independent} for details):

\begin{equation}
\label{eqn:3}
    \begin{split}
    \Iprod{v, \Sigma (\Theta_{1} - \Theta_{2})  } \leq \bigO{1}  \Bigg(&  \underbrace{\expecf{\calG}{\Iprod{v, x \Paren{ x^\top \Theta_{1}  - y}  }} }_{(\textrm{i})} + \underbrace{\expecf{\calG}{\Iprod{v, x' \Paren{ (x')^\top \Theta_{2}  - y'}  }} }_{(\textrm{ii})} \\
    & + \expecf{\calG}{\Iprod{v, x \Paren{  y - x^\top \Theta_{1}}   }  1_{(x,y) \neq (x',y')} } \\
    & + \expecf{\calG}{\Iprod{v, x' \Paren{  y' - (x')^\top \Theta_{2}}   }  1_{(x,y) \neq (x',y')} }    \Bigg)   
\end{split}
\end{equation}
Observe, since we have a maximal coupling, the last two terms appearing in Equation \eqref{eqn:3} are non-zero only on an $\epsilon$-measure support.  To bound them, we decouple the indicator using H\"{o}lder's inequality,

\begin{equation}
\begin{split}
\expecf{\calG }{ \Iprod{v, x (y - x^{\top}\Theta_1)} 1_{(x,y) \neq (x',y')} } & \leq \expecf{}{1_{(x,y) \neq (x',y')} }^{\frac{k-1}{k}}  \expecf{}{\Iprod{v, x}^k \Paren{y - x^{\top} \Theta_1}^k }^{\frac{1}{k}} \\
& \leq \epsilon^{1-1/k} \cdot \underbrace{\expecf{}{\Iprod{v, x}^k \Paren{y - x^{\top} \Theta_1}^k }^{\frac{1}{k}}}_{(\textrm{iii})}
\end{split}
\end{equation}
where we used the maximality of the coupling $\calG$ to bound $\expecf{}{1_{(x,y) \neq (x',y') }} \leq \epsilon$.  The above analysis can be repeated verbatim for the second term in \eqref{eqn:3} as well. Going forward, we focus on bounding terms (i), (ii) and (iii). 

\paragraph{Gradient Conditions.}
\label{para:gradient_condition}
To bound terms (i) and (ii) in Equation \eqref{eqn:3}, we crucially rely on \textit{gradient information} provided by the least-squares objective.
Concretely, a key observation in our information-theoretic proof is that the candidate hyperplanes are locally optimal: given least-squares loss, for $i\in\{1,2\}$ for all vectors $v$,
\[
\Iprod{\nabla  \expecf{x,y\sim\calD_i}{ \Paren{y - x^\top \Theta_i}^2 } , v} =  \expecf{x,y\sim\calD_i}{\Iprod{v, xx^{\top}\Theta_i - xy } } = 0
\]
where $\Theta_1$ and $\Theta_2$ are the optimal hyperplanes for $\calD_1$ and $\calD_2$ respectively.
 Therefore, both (i) and (ii) are identically $0$. It remains to bound (iii).



\paragraph{Independence and Negatively Correlated Moments.}
\label{para:ncm}
We observe that term (iii) can be interpreted as the $k$-th order correlation between the distribution of the covariates projected along $v$ and the distribution of the noise in the linear model. Here, we observe that if the linear model satisfies the \textit{negatively correlated moments} condition (Definition \ref{def:negative_moments}), we can decouple the expectation and bound each term independently:
\begin{equation}
\label{eqn:using_ncm}
\expecf{}{\Iprod{v, x}^k \Paren{y - x^{\top} \Theta_1}^k }^{1/k} \leq \expecf{}{\Iprod{v, x}^k }^{1/k} \expecf{}{\Paren{y - x^{\top} \Theta_1}^k }^{1/k}
\end{equation}
Observe, when the underlying linear model has independent noise, Equation \eqref{eqn:using_ncm} follows for any $k$. We thus crucially exploit the structure of the noise and require a considerably weaker notion than independence. 
Further, if the \textit{negatively correlated moments} property is not satisfied, we can use Cauchy-Schwarz to decouple the expectation in Equation  \eqref{eqn:using_ncm} and incur a $\epsilon^{1-2/k}$ dependence (see Corollary \ref{cor:dependent_noise}). Conceptually, we emphasize that the \textit{negatively correlated moments} condition may be of independent interest to design estimators that exploit independence in various statistics problems.

\paragraph{Hypercontractivity.} 
To bound the RHS in Equation \eqref{eqn:using_ncm}, we use our central distributional assumption of hypercontractive $k$-th moments (Definition \ref{def:hyp}) of the covariates :
\[
\expecf{}{\Iprod{v, x}^k}^{1/k} \leq \sqrt{c_k} \expecf{}{\Iprod{v,x}^2}^{1/2} = \sqrt{c_k} \Iprod{ v , \Sigma v }^{1/2}
\]
We can bound the noise similarly, by assuming that the noise is hypercontractive and this considerably simplifies our statements. However, hypercontractivity of the noise is not a necessary assumption and prior work indeed incurs a term proportional to the $k$-th moment of the noise. Assuming boundedness of the regression vectors, Klivans, Kothari and Meka \cite{DBLP:conf/colt/KlivansKM18} obtained a uniform upper bound on $k$-th moment of the noise by truncating large samples. We note that the same holds for our estimators and we refer the reader to Section 5.2.3 in their paper. 
Finally, substituting $v = \Theta_1 - \Theta_2$ and rearranging, completes the information-theoretic proof. 

We note that our approach already differs from prior work  \cite{DBLP:conf/colt/KlivansKM18, prasad2018robust, zhu2019generalized}
and to our knowledge, we obtain the first information theoretic proof that being $\epsilon$-close in TV distance implies that the optimal hyperplanes are $\bigO{\epsilon^{1-1/k}}$ close in $\ell_2$ distance.
Next, we use this proof as motivation to construct a robust estimator. We do this by explicitly enforcing the gradient condition, negatively correlated moments, and hypercontractivity as constraints in the description of the robust estimator. In contrast, prior work only uses hypercontractivity or the gradient information in the analysis of their robust estimators. We describe these details below.  



\subsection{Proofs to Inefficient Algorithms: Distilling Constraints}
Given an $\epsilon$-corrupted sample generated by Model \ref{model:robust_regression}, a natural approach to design a robust (albeit inefficient) estimator is to search over all subsets of size $(1-\epsilon)n$, minimize the least squares objective on each one and output the hyperplane that achieves the minimal least-squares objective. Klivans, Kothari and Meka \cite{DBLP:conf/colt/KlivansKM18} implicitly analyze this estimator and show that it obtains rate $\epsilon^{1/2 - 1/k}$, which is sub-optimal. Instead, we search over all subsets of size $(1-\epsilon)n$ and in addition to minimizing the least-squares objective, we check if the uniform distribution over the samples is hypercontractive, the corresponding hyperplane satisfies the gradient condition and that the distribution over the covariates and noise satisfies the negatively correlated moments condition. Then, picking the hyperplane that achieves minimal least-squares cost and satisfies the aforementioned constraints obtains the optimal rate.

Since we work in the strong contamination model, where the adversary is allowed to both add and delete samples, there may not be any i.i.d. subset of $(1-\epsilon)n$ samples and thus enforcing constraints directly on the samples does not suffice. 
Instead, we create variables $2n$ variables denoted by $\{ (x'_1,y'_1), \ldots (x'_n,y'_n) \}$ that serve as a proxy for the uncorrupted samples.  We can now enforce constraints on the variables with impunity since a there exists a feasible assignment, namely the uncorrupted samples. With the goal of obtaining an efficiently computable estimator, we restrict to enforcing only polynomial constraints, instead of arbitrary ones.

\paragraph{Intersection Constraints.}
The discrete analogue of the coupling argument is to ensure high intersection between the variables of our polynomial program and the uncorrupted samples. 
We know that at least a $(1-\epsilon)$-fraction of the samples we observe agree with the uncorrupted samples. To this end, we create indicator variables $w_i$, for $i \in [n]$ such that :
\begin{equation*}
  \left \{
    \begin{aligned}
      &&
      \textstyle\sum_{i\in[n]} w_i
      &= (1-\epsilon) n\\
      &\forall i\in [n].
      & w_i^2
      & = w_i \\
      &\forall i\in [n] & w_i(x'_i - x_i) &=0 \\
      &\forall i\in [n] & w_i(y'_i - y_i) &=0 \\
    \end{aligned}
  \right \}
\end{equation*}
The intersection constraints ensure that our polynomial system variables agree with the observed samples on $(1-\epsilon)n$ points. We note that such constraints are now standard in the literature, and indeed are the only constraints explicitly enforced by \cite{DBLP:conf/colt/KlivansKM18}.

\paragraph{Independence as a Polynomial Inequality.}
The central challenge in obtaining optimal rates for robust regression is to leverage the independence of the noise and covariates. 
Since independence is a property of the marginals of $\calD$, it is not immediately clear how to leverage it while designing a robust estimator. 

However, recall that we do not require independence in full generality and use \textit{negatively correlated moments} as a proxy for independence. 
Ideally, we would want to enforce the polynomial inequality corresponding to \textit{negatively correlated moments} directly on the variables of our polynomial program as follows: 
\begin{equation*}
  \left \{
    \begin{aligned}
    & \forall r \leq k/2,
      & \frac{1}{n} \sum_{i \in [n]} \Paren{  v^\top x'_i  \Paren{ y'_i - (x'_i)^{\top} \Theta} }^{2r}  &\leq \bigO{1}  \Paren{\frac{1}{n} \sum_{i \in [n]}  ( v^{\top}x'_i)^{2r} } \Paren{\frac{1}{n} \sum_{i \in [n]} \Paren{y'_i - (x'_i)^{\top} \Theta}^{2r} }\\
    \end{aligned}
  \right \}
\end{equation*}
where $\Theta$ is a variable corresponding to the true hyperplane. To demonstrate feasibility of this constraint, we would require a finite sample analysis, showing that uncorrupted samples from a hypercontractive distribution satisfy the above inequality. Observe, when $r=k/2$, the LHS is a degree-$k$ polynomial and our distribution may be too heavy-tailed to achieve any concentration. 

Instead, we observe that since hypercontractivity is preserved under sampling, we can relax our polynomial constraint by applying hypercontractivity to the terms in the RHS above :
\begin{equation*}
  \left \{
    \begin{aligned}
    & \forall r \leq k/2,
      & \frac{1}{n} \sum_{i \in [n]} \Paren{  v^\top x'_i  \Paren{ y'_i - (x'_i)^{\top} \Theta} }^{2r}  &\leq \bigO{1}  \Paren{\frac{1}{n} \sum_{i \in [n]}  ( v^{\top}x'_i)^{2}}^{r} \Paren{\frac{1}{n} \sum_{i \in [n]} \Paren{y'_i - (x'_i)^{\top} \Theta}^{2} }^r\\
    \end{aligned}
  \right \}
\end{equation*}
In Lemma \ref{lem:soundness} we show that the above inequality is feasible for the uncorrupted samples. In particular, given at least, $d^{\Omega(k)}$  i.i.d. samples from $\calD$, the above inequality holds on the samples with high probability.

Perhaps surprisingly, the dependence on $\epsilon$ achieved when $\calD$ has \textit{negatively correlated moments} matches the information theoretically optimal rate for independent noise. 
We thus expect the notion of \textit{negatively correlated moments} to lead to new estimators for problems where independence of random variables requires to be formulated as a polynomial inequality.

\paragraph{Hypercontractivity Constraints.}
Since hypercontractivity is a already a polynomial identity relating the $k$-th moment to the variance of a distribution, it can easily enforced as a constraint. Conveniently, if the underlying distribution $\calD$ is hypercontractive, then the uniform distribution over a sufficiently large sample is also  sampling also hypercontractive (see Lemma 5.7 in \cite{DBLP:conf/colt/KlivansKM18}). Therefore, the following constraints are feasible:
\begin{equation*}
  \left \{
    \begin{aligned}
      & \forall r \leq k/2
      & \frac{1}{n}\sum_{i \in [n]}  {\langle x'_i , v\rangle^{2r}} 
      & \leq  \left( \frac{c_r}{n} \sum_{i \in [n]}  \langle x'_i, v \rangle^2 \right)^{r}\\
      & \forall r \leq k/2
      & \frac{1}{n}\sum_{i \in [n]}  \Paren{ y'_i - \Iprod{\Theta,x'_i }}^{2r}
      & \leq  \left( \frac{\eta_r }{n} \sum_{i \in [n]}  \Paren{ y'_i - \Iprod{\Theta,x'_i}}^{2} \right)^{r}\\
    \end{aligned}
  \right \}
\end{equation*}
We note that Klivans, Kothari and Meka \cite{DBLP:conf/colt/KlivansKM18} use hypercontractivity of the uncorrupted samples in their analysis but do not explicitly enforce this as a constraint. Enforcing hypercontractivity explicitly on the samples was used by Kothari and Steurer~\cite{kothari2017outlier} in the context of robust moment estimation and Kothari and Steinhardt~\cite{KothariSteinhardt17} in the context of robustly clustering a mixture of spherical Gaussians.

\paragraph{Gradient Constraints.}
Finally, it is crucial in our analysis to enforce that the minimizer we are searching for, $\Theta$, has gradient $0$. For the least-squares loss, the gradient has a simple analytic form : for all $v \in \R^d$,  
\begin{equation*}
  \left \{
    \begin{aligned}
        &&\Iprod{v,\frac{1}{n} \sum_{i \in [n]}   x'_i \Paren{\Iprod{x'_i, \Theta} - y'_i } }^k &=0\\
    \end{aligned}
  \right \}
\end{equation*}
While such optimality conditions are often used in the analysis of estimators (as done in \cite{prasad2018robust}), we emphasize that we hardcode the gradient condition into the description of our robust estimator. 
To the best of our knowledge, no estimator for robust/high-dimensional statistics includes explicit optimality constraints as a part of a polynomial system. 

Solving the least-squares objective on the samples subject to the  polynomial system described by the aforementioned constraints results in an estimator for robust regression that achieves optimal rate. Recall, this follows immidiately from our robust certifiability proof.
However, optimizing polynomial systems can be NP-Hard in the worse case, and thus we briefly describe how to avoid this computational intractability next. 

\subsection{Efficient Algorithms via Sum-of-Squares}
We use the sum-of-squares method to make the aforementioned polynomial system efficiently computable and provide an outline of this approach (see Section \ref{sec:prelims} for a formal treatment of  sum-of-squares proofs). Instead of directly solving the polynomial program, let us instead consider finding a distribution, $\mu$, over feasible solutions $w, x',y'$ and $\Theta$ that minimizes $$\expecf{w,x',y',\Theta \sim \mu}{ \frac{1}{n} \sum_{i \in [n]} \Paren{w_i y'_i - \Iprod{w_i x'_i, \Theta} }^2 }$$ and satisfies the constraints above. Then, it follows from our information-theoretic proof (Theorem \ref{thm:robust_identitifiability_independent}) that 
\begin{equation}
\label{eqn:not_sos_inequality}
\expecf{\mu}{ \Norm{\Sigma^{1/2}\Paren{ \Theta^* - \Theta} }_2} \leq \bigO{\epsilon^{1-1/k}} \err_{\calD}\Paren{ \Theta^* }^{1/2}
\end{equation}
where $\Theta^*$ is the optimal hyperplane. 

We now face two challenges: finding a distribution over solutions is at least as hard as the original problem and we no longer recover a unique hyperplane.
The latter is easy to address by observing that the hyperplane obtained by averaging over the distribution, $\mu$, suffices:
\[
{ \Norm{\Sigma^{1/2}\Paren{ \Theta^* - \expecf{\mu}{\Theta}} }_2} \leq \expecf{\mu}{ \Norm{\Sigma^{1/2}\Paren{ \Theta^* - \Theta} }_2}
\]
where the inequality follows from convexity of the loss.

Following prior works, it is now natural to instead consider searching for a "pseudo-distribution", $\zeta$, over feasible solutions. A pseudo-distribution is an object similar to a real distribution, but relaxed to allow negative mass on its support (see Subsection \ref{subsec:sos_prelims} for a formal treatment). Crucially, a pseudo-distribution over the polynomial program can be computed efficiently by formulating it as a large SDP. To see why this helps, note any polynomial inequality that can be derived using "sum-of-squares" proofs from a set of polynomial constraints using a \textit{low-degree sum-of-squares proof} remains valid if we replace distributions by "pseudo-distribution". 

For instance, if Equation \ref{eqn:not_sos_inequality} were a polynomial inequality in $w,x',y'$ and $\Theta$, obtained by applying simple transformations that admit sum-of-squares proofs, we could replace $\mu$ by $\zeta$, and obtain an efficient estimator. 
However, Equation \ref{eqn:not_sos_inequality} is not a polynomial inequality and the proof outlined in Subsection \ref{subsec:tv_to_regressor} is not a low-degree sum-of-squares proof. Therefore, a central technical contribution of our work is to formulate the right polynomial identity bounding the distance between $\Theta^*$ and $\Theta$ in terms of the least-squares error incurred by by $\Theta^*$, and deriving this bound from the polynomial constraints using a low-degree sum-of-squares proof. We do this in Section \ref{sec:efficient_estimator}.

\subsection{Distribution Families}

We note that our statistical estimator applies to all distributions, $\calD$, that are $(c_k, k)$-hypercontractive and the rate is completely determined by whether $\calD$ has \textit{negatively correlated moments}. In particular, for the important special case of heavy-tailed regression with independent noise, we obtain rate proportional to $\epsilon^{1-1/k}$ for parameter recovery. 

However, similar to prior work on hypercontractive distributions, our efficient estimators apply to a more restrictive class, i.e. \textit{certifiably} hypercontractive distributions (Definition \ref{def:certifiable_hypercontractivity}). Intuitively, this condition captures the criteria that information about degree-$k$ moment upper bounds is "algorithmically accessible". 
Certifiably hypercontractive distributions are a broad class and include affine transformations of isotropic distributions satisfying Poincar\'e inequalities and all strongly log-concave distributions. For a detailed discussion of distributions satisfying Poincar\'e inequalities and their closure properties, we refer the reader to \cite{KothariSteinhardt17, kothari2017outlier}. 

Surprisingly, while we enforce a constraint corresponding to \textit{negatively correlated moments}, we do not require a certifiable variant of this condition. Therefore, our efficient estimators hold for regression instances where the true distribution satisfies this condition, including the special case where the noise is independent from the covariates. Finally, our estimators unify various noise models and imply that even in the agnostic setting, the rate degrades only when the noise is positively correlated with the covariates.

\section{Preliminaries}
\label{sec:prelims}
Throughout this paper, for a vector $v$, we use $\norm{v}_2$ to denote the Euclidean norm of $v$. For a $n \times m$ matrix $M$, we use $\norm{M}_2 = \max_{\norm{x}_2=1} \norm{Mx}_2$ to denote the spectral norm of $M$ and $\norm{M}_F = \sqrt{\sum_{i,j} M_{i,j}^2}$ to denote the Frobenius norm of $M$. For symmetric matrices we use $\succeq$ to denote the PSD/Loewner ordering over eigenvalues of $M$. Recall, the definition of total variation distance between probability measures: 

\begin{definition}[Total Variation Distance]
\label{def:tv}
The TV distance between distributions with PDFs $p,q$ is defined as $\frac{1}{2}\int_{-\infty}^{\infty}|p(x)-q(x)|dx$.
\end{definition}

Given a distribution $\calD$ over $\R^d \times \R$, we consider the least squares error of a vector $\Theta$ w.r.t. $\calD$ to be  $\err_{\calD}(\Theta) = \expecf{x,y \sim \calD}{ \Paren{y - \Iprod{x, \Theta}}^2}$. The linear regression problem minimizes the error over all $\Theta$. The minimizer, $\Theta_{\calD}$ of the aformentioned error satisfies the following "gradient condition" : for all $v \in \R^d$, 

\[
\expecf{x, y \sim \calD}{ \Iprod{ v, x x^{\top} \Theta_\calD - xy } } = 0
\]

\begin{fact}[Convergence of Empirical Moments, implicit in Lemma 5.5 \cite{kothari2017outlier} ]
\label{fact:convergnce_of_empirical_moments}
Let $\calD$ be a $(c_k,k)$-hypercontractive distribution with covariance $\Sigma$ and let $\calX = \{x_1, \ldots x_n \}$ be $n = \Omega( (d\log(d)/\delta)^{k/2})$  i.i.d. samples from $\calD$. Then, with probability at least $1-\delta$, 
\begin{equation*}
    (1-0.1)\Sigma \preceq \frac{1}{n} \sum_i^n x_i x_i^{\top} \preceq (1+0.1) \Sigma
\end{equation*}
\end{fact}

\begin{fact}[TV Closeness to Covariance Closeness, Lemma 2.2 \cite{kothari2017outlier}]
\label{fact:hypercontractive_covariance}
Let $\calD_1, \calD_2$ be $(c_k,k)$-hypercontractive distributions over $\R^d$ such that $\Norm{\calD - \calD'}_{\tv} \leq \epsilon$, where $0<\epsilon < \bigO{(1/c_k)^{\frac{k}{k-1}}}$. Let $\Sigma_1, \Sigma_2$ be the corresponding covariance matrices. Then, for $\delta \leq \bigO{c_k~\epsilon^{1-1/k}} < 1$,
\begin{equation*}
    (1-\delta)\Sigma_2 \preceq \Sigma_1 \preceq (1+\delta) \Sigma_2
\end{equation*}
\end{fact}

\begin{lemma}[Löwner Ordering for Hypercontractive Samples]
\label{lem:hypercontractive_lowner_sampling}
Let $\calD$ be a $(c_k, k)$-hypercontractive distribution with covariance $\Sigma$ and and let $\calU$ be the uniform distribution over $n$ samples. Then, with probability $1-\delta$, 
\[
\Norm{ \Sigma^{-1/2}\hat\Sigma\Sigma^{-1/2} - I }_F \leq \frac{C_4d^2}{\sqrt{n}\sqrt{\delta}},
\]
where $\hat\Sigma = \frac{1}{n}\sum_{i\in [n]} x_i x_i^{\top}$. 
\end{lemma}

Next, we define the technical conditions required for efficient estimators. Formally, 

\begin{definition}[Certifiable Hypercontractivity]
\label{def:certifiable_hypercontractivity}
A distribution $\calD$ on $\R^d$ is $(c_k, k)$-certifiably hypercontractive if for all $r\leq k/2$, there exists a degree $\bigO{k}$ sum-of-squares proof (defined below) of the following inequality in the variable $v$
\begin{equation*}
    \expecf{x\sim \calD }{ \Iprod{x,v}^{2r} } \leq  \expecf{x\sim\calD}{ c_r \Iprod{x,v}^2 }^{r}
\end{equation*}
such that $c_r \leq c_k$.
\end{definition}

Next, we note that if a distribution $\calD$ is certifiably hypercontractive, the uniform distribution over $n$ i.i.d. samples from $\calD$ is also certifiably hypercontractive. 

\begin{fact}[Sampling Preserves Certifiable Hypercontractivity, Lemma 5.5 \cite{kothari2017outlier} ]
\label{fact:sampling_preserves_hypercontract}
Let $\calD$ be a $(c_k, k)$-certifiably hypercontractive distribution on $\R^d$. Let $\calX$ be a set of $n = \Omega\Paren{ \Paren{d\log(d/\delta)}^{k/2}/\gamma^2}$ i.i.d. samples from $\calD$. Then, with probability $1-\delta$, the uniform distribution over $\calX$ is $(c_k + \gamma, k)$-certifiably hypercontractive. 
\end{fact}

We also note that certifiably hypercontractivity is preserved under Affine transformations of the distribution. 

\begin{fact}[Certifiable Hypercontractivity under Affine Transformations, Lemma 5.1, 5.2 \cite{kothari2017outlier}]
\label{fact:certifiable_hypercontractive_under_affine_transformations}
Let $x \in \R^d$ be a random variable drawn from a $(c_k, k)$-certifiably hypercontractive distribution. Then, for matrix $A$ and vector $b$, the distribution over the random variable $Ax + b$ is also $(c_k, k)$-certifiably hypercontractive. 
\end{fact}

Next, we formally define the condition on the moments and noise that we require to obtain efficient algorithms. We note that for technical reasons it is not simply a polynomial identity encoding Definition \ref{def:negative_moments}.

\begin{definition}[Certifiable Negatively Correlated Moments]
\label{def:certifiable_ncm}
A distribution $\calD$ on $\R^d \times \R$ has $\bigO{1}$-certifiable negatively correlated moments if for all $r \leq k/2$ there exists a degree $\bigO{k}$ sum-of-squares proof of the following inequality 
\begin{equation*}
    \expecf{x,y\sim\calD}{ \Paren{  v^\top x  \Paren{ y - x^{\top} \Theta} }^{2r} } \leq \bigO{\lambda_r^r }  \Paren{ \expecf{}{  ( v^{\top}x)^{2}}^{r}}  \Paren{ \expecf{}{ \Paren{y - x^{\top} \Theta}^{2} }^{r} }
\end{equation*}
for a fixed vector $\Theta$. 
\end{definition}

\subsection{SoS Background.}
\label{subsec:sos_prelims}

In this subsection, we provide the necessary background for the sum-of-squares proof system. We follow the exposition as it appears in lecture notes by Barak~\cite{barak2016proofs}, the Appendix of Ma, Shi and Steurer~\cite{ma2016polynomial}, and the preliminary sections of several recent works~\cite{kothari2017outlier,karmalkar2019list,klivans2018efficient, bakshi2020list, bakshi2020outlier}.   

\paragraph{Pseudo-Distributions.}
We can represent a discrete probability distribution over $\R^n$ by its probability mass function $D\from \R^n \to \R$ such that $D \geq 0$ and $\sum_{x \in \mathrm{supp}(D)} D(x) = 1$.
Similarly, we can describe a pseudo-distribution by its mass function by relaxing the non-negativity constraint, while still passing certain low-degree non-negativity tests.

\begin{definition}[Pseudo-distribution]
A level-$\ell$ pseudo-distribution is a finitely-supported function $D:\R^n \rightarrow \R$ such that $\sum_{x} D(x) = 1$ and $\sum_{x} D(x) f(x)^2 \geq 0$ for every polynomial $f$ of degree at most $\ell/2$, where the summation is over all $x$ in the support of $D$.
\end{definition}
Next, we define the notion of pseudo-expectation.
\begin{definition}[Pseudo-expectation]
The pseudo-expectation of a function $f$ on $\R^d$ with respect to a pseudo-distribution $D$, denoted by $\pE_{D(x)}[f(x)]$,  is defined as
\begin{equation}
  \pE_{D(x)}[f(x)] = \sum_{x} D(x) f(x) \,
\end{equation}
\end{definition}
We use the notation $\pE_{D(x)}[(1,x_1, x_2,\ldots, x_n)^{\otimes \ell}]$ to denote the degree-$\ell$ moment tensor of the pseudo-distribution $D$.
In particular, each entry in the moment tensor corresponds to the pseudo-expectation of a monomial of degree at most $\ell$ in $x$. 
Crucially, there's an efficient separation oracle for moment tensors of pseudo-distributions. 

\begin{fact}[\cite{MR939596-Shor87,parrilo2000structured,MR1748764-Nesterov00,MR1846160-Lasserre01}]
 \label{fact:sos-separation-efficient}
  For any $n,\ell \in \N$, the following set has a $n^{O(\ell)}$-time weak separation oracle (in the sense of \cite{MR625550-Grotschel81}):
  \begin{equation}
    \Set{ \pE_{D(x)} (1,x_1, x_2, \ldots, x_n)^{\otimes \ell } \mid \text{ degree-$\ell$ pseudo-distribution $D$ over $\R^n$}}\, 
  \end{equation}
\end{fact}
This fact, together with the equivalence of weak separation and optimization \cite{MR625550-Grotschel81} forms the basis of the sum-of-squares algorithm, as it allows us to efficiently approximately optimize over pseudo-distributions. 

Given a system of polynomial constraints, denoted by $ \calA$, we say that it is \emph{explicitly bounded} if it contains a constraint of the form $\{ \|x\|^2 \leq M\}$. Then, the following fact follows from  Fact \ref{fact:sos-separation-efficient} and \cite{MR625550-Grotschel81}:

\begin{fact}[Efficient Optimization over Pseudo-distributions]
There exists an $(n+ m)^{O(\ell)} $-time algorithm that, given any explicitly bounded and satisfiable system\footnote{Here, we assume that the bit complexity of the constraints in $ \calA$ is $(n+m)^{O(1)}$.} $ \calA$ of $m$ polynomial constraints in $n$ variables, outputs a level-$\ell$ pseudo-distribution that satisfies $ \calA$ approximately. \label{fact:eff-pseudo-distribution}
\end{fact}

We now define basic facts for pseudo-distributions, which extend facts that hold for standard probability distributions, which can be found in several prior works listed above.

\begin{fact}[Cauchy-Schwarz for Pseudo-distributions]
Let $f,g$ be polynomials of degree at most $d$ in indeterminate $x \in \R^d$. Then, for any degree d pseudo-distribution $\tzeta$,
$\pE_{\tzeta}[fg] \leq \sqrt{\pE_{\tzeta}[f^2]} \sqrt{\pE_{\tzeta}[g^2]}$.
 \label{fact:pseudo-expectation-cauchy-schwarz}
\end{fact}

\begin{fact}[Hölder's Inequality for Pseudo-Distributions] \label{fact:pseudo-expectation-holder}
Let $f,g$ be polynomials of degree at most $d$ in indeterminate $x \in \R^d$. 
Fix $t \in \N$. Then, for any degree $dt$ pseudo-distribution $\tzeta$,
$\pE_{\tzeta}[f^{t-1}g] \leq \paren{\pE_{\tzeta}[f^t]}^{\frac{t-1}{t}} \paren{\pE_{\tzeta}[g^t]}^{1/t}$. In particular, for all even integers $k $,
$\pE_{\tzeta}[f]^k \leq\pE_{\tzeta}[f^k]$.
\end{fact}



\paragraph{Sum-of-squares proofs.}

Let $f_1, f_2, \ldots, f_r$ and $g$ be multivariate polynomials in $x$.
A \emph{sum-of-squares proof} that the constraints $\{f_1 \geq 0, \ldots, f_m \geq 0\}$ imply the constraint $\{g \geq 0\}$ consists of  polynomials $(p_S)_{S \subseteq [m]}$ such that
\begin{equation}
g = \sum_{S \subseteq [m]} p^2_S \cdot \Pi_{i \in S} f_i
\end{equation}
We say that this proof has \emph{degree $\ell$} if for every set $S \subseteq [m]$, the polynomial $p^2_S \Pi_{i \in S} f_i$ has degree at most $\ell$.
If there is a degree $\ell$ SoS proof that $\{f_i \geq 0 \mid i \leq r\}$ implies $\{g \geq 0\}$, we write:
\begin{equation}
  \{f_i \geq 0 \mid i \leq r\} \sststile{\ell}{}\{g \geq 0\}
\end{equation}
For all polynomials $f,g\colon\R^n \to \R$ and for all functions $F\colon \R^n \to \R^m$, $G\colon \R^n \to \R^k$, $H\colon \R^{p} \to \R^n$ such that each of the coordinates of the outputs are polynomials of the inputs, we have the following inference rules.

The first one derives new inequalities by addition/multiplication:
\begin{equation} \label{eq:sos-addition-multiplication-rule}
\frac{ \calA \sststile{\ell}{} \{f \geq 0, g \geq 0 \} } { \calA \sststile{\ell}{} \{f + g \geq 0\}}, \frac{ \calA \sststile{\ell}{} \{f \geq 0\},  \calA \sststile{\ell'}{} \{g \geq 0\}} { \calA \sststile{\ell+\ell'}{} \{f \cdot g \geq 0\}} \tag{Addition/Multiplication Rule} 
\end{equation}
The next one derives new inequalities by transitivity: 
\begin{equation} \label{eq:sos-transitivity}
\frac{ \calA \sststile{\ell}{}  \calB,  \calB \sststile{\ell'}{} C}{ \calA \sststile{\ell \cdot \ell'}{} C}   \tag{Transitivity Rule} 
\end{equation}
Finally, the last rule derives new inequalities via substitution:
\begin{equation} \label{eq:sos-substitution}
\frac{\{F \geq 0\} \sststile{\ell}{} \{G \geq 0\}}{\{F(H) \geq 0\} \sststile{\ell \cdot \deg(H)} {} \{G(H) \geq 0\}}\tag{Substitution Rule} 
\end{equation}

Low-degree sum-of-squares proofs are sound and complete if we take low-level pseudo-distributions as models.
Concretely, sum-of-squares proofs allow us to deduce properties of pseudo-distributions that satisfy some constraints.
\begin{fact}[Soundness]
  \label{fact:sos-soundness}
  If $D \sdtstile{r}{}  \calA$ for a level-$\ell$ pseudo-distribution $D$ and there exists a sum-of-squares proof $ \calA \sststile{r'}{}  \calB$, then $D \sdtstile{r\cdot r'+r'}{}  \calB$.
\end{fact}
If the pseudo-distribution $D$ satisfies $ \calA$ only approximately, soundness continues to hold if we require an upper bound on the bit-complexity of the sum-of-squares $ \calA \sststile{r'}{} B$  (number of bits required to write down the proof). In our applications, the bit complexity of all sum of squares proofs will be $n^{O(\ell)}$ (assuming that all numbers in the input have bit complexity $n^{O(1)}$). This bound suffices in order to argue about pseudo-distributions that satisfy polynomial constraints approximately.

The following fact shows that every property of low-level pseudo-distributions can be derived by low-degree sum-of-squares proofs.
\begin{fact}[Completeness]
  \label{fact:sos-completeness}
  Suppose $d \geq r' \geq r$ and $ \calA$ is a collection of polynomial constraints with degree at most $r$, and $ \calA \vdash \{ \sum_{i = 1}^n x_i^2 \leq B\}$ for some finite $B$.

  Let $\{g \geq 0 \}$ be a polynomial constraint.
  If every degree-$d$ pseudo-distribution that satisfies $D \sdtstile{r}{}  \calA$ also satisfies $D \sdtstile{r'}{} \{g \geq 0 \}$, then for every $\epsilon > 0$, there is a sum-of-squares proof $ \calA \sststile{d}{} \{g \geq - \epsilon \}$.
\end{fact}

\paragraph{Basic Sum-of-Squares Proofs.} Next, we use the following basic facts regarding sum-of-squares proofs. For further details, we refer the reader to a recent monograph~\cite{fleming2019semialgebraic}.

\begin{fact}[Operator norm Bound]
\label{fact:operator_norm}
Let $A$ be a symmetric $d\times d$ matrix and $v$ be a vector in $\mathbb{R}^d$. Then,
\[
\sststile{2}{v} \Set{ v^{\top} A v \leq \|A\|_2\|v\|^2_2 }.
\]
\end{fact}

\begin{fact}[SoS Almost Triangle Inequality] \label{fact:sos-almost-triangle}
Let $f_1, f_2, \ldots, f_r$ be indeterminates. Then,
\[
\sststile{2t}{f_1, f_2,\ldots,f_r} \Set{ \Paren{\sum_{i\leq r} f_i}^{2t} \leq r^{2t-1} \Paren{\sum_{i =1}^r f_i^{2t}}}.
\]
\end{fact}

\begin{fact}[SoS Hölder's Inequality]\label{fact:sos-holder}
Let $w_1, \ldots w_n$ be indeterminates and let $f_1,\ldots f_n$ be polynomials of degree $m$ in vector valued variable $x$. 
Let $k$ be a power of 2.  
Then, 
\[
\Set{w_i^2 = w_i, \forall i\in[n] } \sststile{2km}{x,w} \Set{  \Paren{\frac{1}{n} \sum_{i = 1}^n w_i f_i}^{k} \leq \Paren{\frac{1}{n} \sum_{i = 1}^n w_i}^{k-1} \Paren{\frac{1}{n} \sum_{i = 1}^n f_i^k}}. 
\]
\end{fact}

We also use the following fact that allows us to cancel powers of indeterminates within the sum-of-squares proof system.

\begin{fact}[Cancellation Within SoS, Lemma 9.4 \cite{bakshi2020outlier} ]
Let $a,C$ be indeterminates. Then, 
\[
\Set{a \geq 0 } \cup \Set{a^{2t} \leq C a^t } \sststile{2t}{a,C} \Set{a^{2t} \leq C^{2}}.
\]
\label{lem:cancellation-sos}
\end{fact}

\section{Robust Certifiability and Information Theoretic Estimators}

In this section, we provide an estimator that obtains the information theoretically optimal rate for robust regression. We note that we consider the setting where both the covariates and the noise are hypercontractive and the are independent of each other. This setting displays all the key ideas of our estimator. Further, our estimator extends to the remaining settings, such as bounded dependent noise, by simple modifications to the subsequent analysis.

\begin{theorem}[Robust Certifiability with Optimal Rate]
\label{thm:robust_identitifiability_independent}
Given $\epsilon >0$, let $\calD, \calD'$ be distributions over $\R^d \times \R$ such that the respective marginal distributions over $\R^d$, denoted by $\calD_\calX, \calD'_\calX$, are $(c_k, k)$-hypercontractive and $\Norm{ \calD - \calD' }_{\tv} \leq \epsilon$. Let $\calR_{\calD}(\epsilon, \Sigma_{\calD}, \Theta_{\calD})$ and $\calR_{\calD'}(\epsilon, \Sigma_{\calD'}, \Theta_{\calD'})$ be the corresponding instances of robust regression such that $\calD, \calD'$ have negatively correlated moments. Further, for $(x,y)\sim\calD, \calD'$, let the marginal distribution over $y - \Iprod{x, \expecf{}{x x^{\top}}^{-1} \expecf{}{xy}} $ be $(\eta_k, k)$-hypercontractive 
Then, 
\[ 
\Norm{\Sigma_{\calD}^{1/2}(\Theta_{\calD}-\Theta_{\calD'})}_{2} \leq \bigO{ \sqrt{c_k~\eta_k} ~\epsilon^{1-1/k}} \Paren{ \err_{\calD}(\Theta_\calD)^{1/2}  + \err_{\calD'}(\Theta_{\calD'})^{1/2} }
\]
Further, 
\[
\err_{\calD}(\Theta_{\calD'}) \leq \Paren{ 1+\bigO{c_k~\eta_k ~\epsilon^{2-2/k}} } \err_{\calD}(\Theta_\calD)  + \bigO{c_k~\eta_k ~\epsilon^{2-2/k}} \err_{\calD'}(\Theta_{\calD'})
\]
\end{theorem}

\begin{proof}
Consider a maximal coupling of $\calD, \calD'$ over $(x,y) \times (x',y')$, denoted by $\calG$, such that the marginal of $\calG$ $(x,y)$ is $\calD$, the marginal on $(x', y')$ is $\calD'$ and $\mathbb{P}_{\calG}[ \indic{(x,y) = (x',y')}] = 1-\epsilon$. Then, for all $v$, 
\begin{equation}
\label{eqn:main_bound}
\begin{split}
    \Iprod{v, \Sigma_{\calD} (\Theta_{\calD} - \Theta_{\calD'})  } & = \expecf{\calG}{\Iprod{v, x x^{\top} (\Theta_{\calD} - \Theta_{\calD'}) + xy - xy  }} \\
    & =  \expecf{\calG}{\Iprod{v, x \Paren{ \Iprod{x, \Theta_{\calD}}  - y}  }} + \expecf{\calG}{\Iprod{v, x \Paren{  y - \Iprod{x, \Theta_{\calD'}} }  }} 
\end{split}
\end{equation}

Since $\Theta_\calD$ is the minimizer for the least squares loss, we have the following gradient condition : for all $v \in \mathbb{R}^d$,
\begin{equation}
    \label{eqn:grad}
    \expecf{(x,y)\sim \calD}{ \Iprod{v,  ( \Iprod{x, \Theta_{\calD}} - y) x } } = 0
\end{equation}
Since $\calG$ is a coupling, using the gradient condition \eqref{eqn:grad} and using that $1 = \indic{(x,y) = (x',y')}$ $+ \indic{(x,y) \neq (x',y')}$, we can rewrite equation \eqref{eqn:main_bound} as

\begin{equation}
\label{eqn:main_bound_restated}
\begin{split}
    \Iprod{v, \Sigma_{\calD} (\Theta_{\calD} - \Theta_{\calD'})  }
    & =   \expecf{\calG}{\Iprod{v, x \Paren{  y - \Iprod{x, \Theta_{\calD'}}}} \indic{(x,y) = (x',y')} } \\
    & \hspace{0.15in} +  \expecf{\calG}{\Iprod{v, x \Paren{  y - \Iprod{x, \Theta_{\calD'}}}} \indic{(x,y) \neq (x',y')} }    \\
    & = \expecf{\calG}{\Iprod{v, x' \Paren{  y' -\Iprod{x', \Theta_{\calD'}}}} \indic{(x,y) = (x',y')} } \\
    & \hspace{0.15in} +  \expecf{\calG}{\Iprod{v, x \Paren{  y - \Iprod{x, \Theta_{\calD'}}}} \indic{(x,y) \neq (x',y')} }    \\
\end{split}
\end{equation}
Consider the first term in the last equality above. Using the gradient condition for $\Theta_{\calD'}$ along with Hölder's Inequality, we have

\begin{equation}
\label{eqn:hypercontactivity_d'}
\begin{split}
     \Big|\mathbb{E}_{\calG} \Big[ &\Iprod{v, x'  \Paren{  y' -\Iprod{x', \Theta_{\calD'}}}}  \indic{(x,y) = (x',y')}  \Big]\Big| \\
    & =  \Big|\expecf{\calD'}{\Iprod{v, x' \Paren{  y' -\Iprod{x', \Theta_{\calD'}}}}  }
     - \expecf{\calG}{\Iprod{v, x' \Paren{  y' -\Iprod{x', \Theta_{\calD'}}}} \indic{(x,y) \neq (x',y')} } \Big| \\
    & = \left|\expecf{\calG}{\Iprod{v, x' \Paren{  y' -\Iprod{x', \Theta_{\calD'}}}} \indic{(x,y) \neq (x',y')} } \right| \\ 
    & \leq \left|\expecf{\calG}{ \indic{(x,y) \neq (x',y')}^{k/(k-1)} }^{(k-1)/k} \right| \cdot \left|\expecf{\calD'}{ \Iprod{v, x' \Paren{  y' -\Iprod{x', \Theta_{\calD'}}}}^k  }^{1/k} \right| \\ 
\end{split}
\end{equation}
Observe, since $\calG$ is a maximal coupling $\expecf{\calG}{ \indic{(x,y) \neq (x',y')}}^{(k-1)/k} \leq \epsilon^{1-1/k}$. Further, since $\calD'$ has negatively correlated moments,  
\[
\expecf{\calD'}{ \Iprod{v, x'}^k \cdot  \Paren{  y' -\Iprod{x', \Theta_{\calD'}}}^k  }  =\expecf{\calD'}{ \Iprod{v, x'}^k } \expecf{\calD'}{  \Paren{  y' -\Iprod{x', \Theta_{\calD'}}}^k  }
\]
By hypercontractivity of the covariates and the noise, we have 

\[
\expecf{\calD'}{ \Iprod{v, x'}^k }^{1/k} \expecf{\calD'}{  \Paren{  y' -\Iprod{x', \Theta_{\calD'}}}^k  }^{1/k} \leq \bigO{ \sqrt{c_k ~ \eta_k} } \Paren{ v^{\top}\Sigma_{\calD'}v}^{1/2} \expecf{x',y'\sim\calD'}{\Paren{y' - \Iprod{x', \Theta_{\calD'}} }^2 }^{1/2}
\]
Therefore, we can restate \eqref{eqn:hypercontactivity_d'} as follows 
\begin{equation}
\label{eqn:hypercontractive_bound_d'_noise}
\begin{split}
\Big| \expecf{\calG}{\Iprod{v, x'  \Paren{  y' -\Iprod{x', \Theta_{\calD'}}}}  \indic{(x,y) = (x',y')}  }\Big| & \leq \bigO{ \sqrt{c_k ~ \eta_k} ~\epsilon^{\frac{k-1}{k}} } \Paren{ v^{\top}\Sigma_{\calD'}v}^{\frac{1}{2}} \\
& \hspace{0.15in} \expecf{x',y'\sim\calD'}{\Paren{y' - \Iprod{x', \Theta_{\calD'}} }^2 }^{\frac{1}{2}}
\end{split}
\end{equation}
It remains to bound the second term in the last equality of equation \eqref{eqn:main_bound_restated}, and we proceed as follows : 

\begin{equation}
\label{eqn:hypercontactivity_d}
\begin{split}
    \expecf{\calG}{\Iprod{v, x \Paren{  y - \Iprod{x, \Theta_{\calD'}}  }} \indic{(x,y) \neq (x',y')} } 
    & =  \expecf{\calG}{\Iprod{v, x  x^{\top} \Paren{ \Theta_\calD - \Theta_{\calD'}}} \indic{(x,y) \neq (x',y')} } \\
    & \hspace{0.15in}+ \expecf{\calG}{\Iprod{v, x \Paren{  y - \Iprod{x, \Theta_{\calD}}}} \indic{(x,y) \neq (x',y')} }\\
\end{split}
\end{equation}
We bound the two terms above separately.  Observe, applying Hölder's Inequality to the first term, we have
\begin{equation}
\label{eqn:tricky_term}
    \begin{split}
         \expecf{\calG}{\Iprod{v, x  x^{\top} \Paren{ \Theta_\calD - \Theta_{\calD'}}} \indic{(x,y) \neq (x',y')} } & \leq  \expecf{\calG}{\indic{(x,y) \neq (x',y')} }^{\frac{k-2}{k}}  \expecf{\calG}{\Iprod{v, x  x^{\top} \Paren{ \Theta_\calD - \Theta_{\calD'}}}^{\frac{k}{2}}}^{\frac{2}{k}}\\
         & \leq  ~\epsilon^{\frac{k-2}{k}} \expecf{\calG}{\Iprod{v, x  x^{\top} \Paren{ \Theta_\calD - \Theta_{\calD'}}}^{\frac{k}{2}}}^{\frac{2}{k}} \\
    \end{split}
\end{equation}
To bound the second term in equation \ref{eqn:hypercontactivity_d}, we again use Hölder's Inequality followed $\calD$ having negatively correlated moments,

\begin{equation}
\label{eqn:hypercontract_holder}
    \begin{split}
        \expecf{\calG}{\Iprod{v, x \Paren{  y - \Iprod{x, \Theta_{\calD}}}} \indic{(x,y) \neq (x',y')} } & \leq \expecf{\calG}{ \indic{(x,y) \neq (x',y')} }^{\frac{k-1}{k}}\expecf{\calG}{\Iprod{v, x \Paren{  y - \Iprod{x, \Theta_{\calD}}}}^{k} }^{\frac{1}{k}} \\
        &\leq \epsilon^{\frac{k-1}{k}} \expecf{x\sim \calD}{\Iprod{v,x}^k }^{1/k} \expecf{x, y\sim \calD}{ \Paren{  y - \Iprod{x, \Theta_{\calD}}}^k }^{1/k}  \\
        & \leq  \epsilon^{\frac{k-1}{k}}~ \sqrt{c_k ~ \eta_k} \Paren{ v^{\top} \Sigma_{\calD} v}^{1/2} \expecf{x,y\sim\calD}{ (y -  \Iprod{x, \Theta_{\calD}})^2}^{1/2} 
    \end{split}
\end{equation}
where the last inequality follows from hypercontractivity of the covariates and noise. 
Substituting the upper bounds obtained in Equations \eqref{eqn:tricky_term} and \eqref{eqn:hypercontract_holder} back in to \eqref{eqn:hypercontactivity_d},
\begin{equation*}
\begin{split}
    \expecf{\calG}{\Iprod{v, x \Paren{  y - \Iprod{x, \Theta_{\calD'}}}} \indic{(x,y) \neq (x',y')} } & \leq\epsilon^{\frac{k-2}{k}} \expecf{\calG}{\Iprod{v, x  x^{\top} \Paren{ \Theta_\calD - \Theta_{\calD'}}}^{\frac{k}{2}}}^{\frac{2}{k}} \\
    & \hspace{0.15in} +   \epsilon^{\frac{k-1}{k}}~ \sqrt{c_k ~ \eta_k} \Paren{ v^{\top} \Sigma_{\calD} v}^{1/2} \expecf{x,y\sim\calD}{ (y -  \Iprod{x, \Theta_{\calD}})^2}^{1/2} 
\end{split}
\end{equation*}
Therefore, we can now upper bound both terms in Equation \eqref{eqn:main_bound_restated} as follows: 

\begin{equation}
\label{eqn:main_bound_rerestated}
\begin{split}
    \Iprod{v, \Sigma_{\calD} (\Theta_{\calD} - \Theta_{\calD'})  } & \leq \bigO{ c_k ~ \eta_k ~\epsilon^{\frac{k-1}{k}}} \Paren{ v^{\top}\Sigma_{\calD'}v}^{1/2} \expecf{x',y'\sim\calD'}{\Paren{y' - \Iprod{x', \Theta_{\calD'}} }^2 }^{1/2} \\
    & \hspace{0.15in} + \bigO{\epsilon^{\frac{k-2}{k}}} \expecf{\calG}{\Iprod{v, x  x^{\top} \Paren{ \Theta_\calD - \Theta_{\calD'}}}^{k/2}}^{2/k} \\
    & \hspace{0.15in} + \bigO{\epsilon^{\frac{k-1}{k}}~ \sqrt{c_k ~ \eta_k}} \Paren{ v^{\top} \Sigma_{\calD} v}^{1/2} \expecf{x,y\sim\calD}{ (y -  \Iprod{x, \Theta_{\calD}})^2}^{1/2} 
\end{split}
\end{equation}
Recall, since the marginals of $\calD$ and $\calD'$ on $\R^d$ are $(c_k, k)$-hypercontractive and $\Norm{\calD - \calD' }_{\tv} \leq \epsilon$, it follows from Fact \ref{fact:hypercontractive_covariance} that 
\begin{equation}
\label{eqn:cov_closeness_under_tv}
    \Paren{1- 0.1}\Sigma_{\calD'} \preceq \Sigma_{\calD} \preceq \Paren{1+0.1}\Sigma_{\calD'}
\end{equation}
when $\epsilon \leq \bigO{(1/c_k k)^{k/k-1}}$.
Now, consider the substitution $v = \Theta_\calD - \Theta_{\calD'}$. Observe, 
\begin{equation}
\label{eqn:hypercontractivity_along_regressor}
\begin{split}
\expecf{\calG}{\Iprod{v, x  x^{\top} \Paren{ \Theta_\calD - \Theta_{\calD'}}}^{k/2}}^{2/k} & = \expecf{\calD}{\Iprod{ x,   \Paren{ \Theta_\calD - \Theta_{\calD'}}}^{k}}^{2/k} \\
& \leq c_k  \Norm{\Sigma^{1/2}_{\calD} (\Theta_{\calD} - \Theta_{\calD'}) }^2_2
\end{split}
\end{equation}
Then, using the bounds in \eqref{eqn:cov_closeness_under_tv} and \eqref{eqn:hypercontractivity_along_regressor} along with $v = \Theta_{\calD} - \Theta_{\calD'}$ in 
Equation \ref{eqn:main_bound_rerestated}, we have
\begin{equation}
\label{eqn:main_bound_rererestated}
\begin{split}
    \Paren{1- \bigO{\epsilon^{\frac{k-2}{k}} c_k } } \Norm{\Sigma^{1/2}_{\calD} (\Theta_{\calD} - \Theta_{\calD'}) }^2_2 & \leq \bigO{ \sqrt{c_k ~ \eta_k} ~\epsilon^{\frac{k-1}{k}}} \Norm{\Sigma^{1/2}_{\calD} (\Theta_{\calD} - \Theta_{\calD'}) }_2  \\
    & \hspace{0.15in} \Paren{ \expecf{x',y'\sim\calD'}{\Paren{y' - \Iprod{x', \Theta_{\calD'}} }^2 }^{\frac{1}{2}} + \expecf{x,y\sim\calD}{ (y -  \Iprod{x, \Theta_{\calD}})^2}^{\frac{1}{2}} } \\
\end{split}
\end{equation}
Dividing out \eqref{eqn:main_bound_rererestated} by $\Paren{1- \bigO{\epsilon^{\frac{k-2}{k}} c_k } } \Norm{\Sigma^{1/2}_{\calD} (\Theta_{\calD} - \Theta_{\calD'}) }^2_2$ and observing that $\bigO{\epsilon^{\frac{k-2}{k}} c_k }$ is upper bounded by a fixed constant less than $1$ yields the parameter recovery bound.

Given the parameter recovery result above, we bound the least-squares loss between the two hyperplanes on $\calD$ as follows: 
\begin{equation}
    \begin{split}
        \big| \err_{\calD}(\Theta_{\calD}) - \err_{\calD}(\Theta_{\calD'}) \big| & = \Big| \expecf{(x,y)\sim \calD}{\Paren{ y - x^{\top}\Theta_{\calD} }^2   - {\Paren{ y - x^{\top}\Theta_{\calD'} + x^{\top}\Theta_{\calD} - x^{\top}\Theta_{\calD}   }^2 } }  \Big| \\
        & = \Big| \expecf{(x,y)\sim \calD}{ \Iprod{  x,(\Theta_{\calD} - \Theta_{\calD'}) }^2 + 2(y- x^{\top}\Theta_{\calD})x^{\top}(\Theta_{\calD} - \Theta_{\calD'})  } \Big|\\
        & \leq \bigO{ c_k~\eta_k~\epsilon^{2-2/k} }\Paren{ \expecf{x',y'\sim\calD'}{\Paren{y' - \Iprod{x', \Theta_{\calD'}} }^2 } + \expecf{x,y\sim\calD}{ (y -  \Iprod{x, \Theta_{\calD}})^2} } 
    \end{split}
\end{equation}
where the last inequality follows from observing $\expecf{}{\Iprod{ \Theta_{\calD} - \Theta_{\calD'}, x (y- x^{\top}\Theta_{\calD})} } = 0$ (gradient condition) and squaring the parameter recovery bound. 
\end{proof}

Next, we consider the setting where the noise is allowed to dependent arbitrarily on the covariates, which captures the well-studied agnostic model. With a slightly modification in our certifiability proof above (using Cauchy-Schwarz instead of independence), we obtain the optimal rate in this setting. We defer the details to Appendix \ref{sec:identifiability_dependent_noise}.   

\begin{corollary}[Robust Regression with Dependent Noise]
\label{cor:dependent_noise}
Let $\calD, \calD'$ be distributions over $\R^d \times \R$ and let $\calR_{\calD}(\epsilon, \Sigma_\calD, \Theta_\calD)$, $\calR_{\calD'}(\epsilon, \Sigma_{\calD'}, \Theta_{\calD'})$ be robust regression instances satisfying the hypothesis in Theorem \ref{thm:robust_identitifiability_independent} such that the negatively correlated moments condition is not satisfied. 
Then, 
\[ 
\Norm{\Sigma_{\calD}^{1/2}(\Theta_{\calD}-\Theta_{\calD'})}_{2} \leq \bigO{\sqrt{c_k ~ \eta_k} ~\epsilon^{1-2/k}} \Paren{ \err_{\calD}(\Theta_\calD)^{1/2}  + \err_{\calD'}(\Theta_{\calD'})^{1/2}}
\]
Further, 
\[
\err_{\calD}(\Theta_{\calD'}) \leq \Paren{ 1+\bigO{c_k~\eta_k ~\epsilon^{2-4/k}} } \err_{\calD}(\Theta_\calD)  +\bigO{c_k~\eta_k ~\epsilon^{2-4/k}}\err_{\calD'}(\Theta_{\calD'})
\]
\end{corollary}

\section{Robust Regression in Polynomial Time}
\label{sec:efficient_estimator}
In this section, we describe an algorithm to compute our robust estimator for linear regression efficiently. We consider a polynomial system that encodes our robust estimator. We then consider a sum-of-squares relaxation of this program and compute an approximately optimal solution for our relaxation. To analyze our algorithm, we consider the dual of the sum-of-squares relaxation and show that the sum-of-squares proof system caputures a variant of our robust identifiability proof.

We begin by recalling notation: let $\calD$ be a distribution over $\R^d \times \R$ such that it is $(\lambda_k, k)$-certifiably hypercontractive. 
Let $\calX = \{(x^*_1,y^*_1), (x^*_2, y^*_2) \ldots (x^*_n, y^*_n) \}$ denote $n$ uncorrupted i.i.d samples from $\calD$ and let $\calX_{\epsilon}=. \{(x_1,y_1), (x_2, y_2) \ldots (x_n, y_n) \}$ be an $\epsilon$-corruption of the samples $\calX$, drawn from a Robust Regression model,  $\calR_{\calD}(\epsilon, \Sigma^*, \Theta^*)$ (Model  \ref{model:robust_regression}). We consider a polynomial system in the variables $\calX' = \{(x'_1,y'_1), (x'_2, y'_2) \ldots (x'_n, y'_n) \}$ and  $w_1, w_2, \ldots w_n \in \{0,1\}^n$ as follows: 

\begin{equation*}
\calA_{\epsilon, \lambda_k}\colon
  \left \{
    \begin{aligned}
      &&
      \textstyle\sum_{i\in[n]} w_i
      &= (1-\epsilon) n\\
      &\forall i\in [n].
      & w_i^2
      & = w_i \\
      &\forall i\in [n] & w_i(x'_i - x_i) &=0 \\
      &\forall i\in [n] & w_i(y'_i - y_i) &=0 \\
      &&\Iprod{v,\frac{1}{n} \sum_{i \in [n]}   x'_i \Paren{\Iprod{x'_i, \Theta} - y_i } }^k &=0\\
      & \forall r \leq k/2
      & \frac{1}{n}\sum_{i \in [n]}  {\langle x'_i , v\rangle^{2r}} 
      & \leq  \left( \frac{\lambda_r}{n} \sum_{i \in [n]}  \langle x'_i, v \rangle^2 \right)^{r}\\
      & \forall r \leq k/2
      & \frac{1}{n}\sum_{i \in [n]}  \Paren{ y'_i - \Iprod{\Theta,x'_i }}^{2r}
      & \leq  \left( \frac{\lambda_r }{n} \sum_{i \in [n]}  \Paren{ y'_i - \Iprod{\Theta,x'_i}}^{2} \right)^{r}\\
      & \forall r \leq k/2
      & \expecf{}{\Paren{  v^\top x'_i  \Paren{ y'_i - (x'_i)^{\top} \Theta} }^{2r}}  &\leq \bigO{\lambda_r^{2r} }  \expecf{}{ \Iprod{v,x'_i}^{2}}^{r} \expecf{}{  \Paren{y'_i - \Iprod{ x'_i,  \Theta}}^{2}}^{r}\\
    \end{aligned}
  \right \}
\end{equation*}

We show that optimizing an appropriate convex function subject to the aforementioned constraint system results in an efficiently computable robust estimator for regression, achieving the information-theoretically optimal rate. Formally, 

\begin{theorem}[Robust Regression with Negatively Correlated Moments, Theorem \ref{thm:informal_negative_moment_polytime} restated]
\label{thm:optimal_efficient_robust_regression}
Given $k\in \mathbb{N}$, $\epsilon>0$ and $n \geq n_0$ samples $\calX_{\epsilon} = \{(x_1, y_1), \ldots (x_n, y_n)\}$ from  $\mathcal{R}_{\calD}(\epsilon,  \Sigma^*, \Theta^*)$, where $\calD$ is a $(\lambda_k,k)$-certifiably hypercontractive distribution over $\R^{d}\times \R$.
Further, $\calD$ has certifiable negatively correlated moments.
Then,  Algorithm \ref{algo:rounding-for-pseudo-distribution} runs in $n^{\mathcal{O}(k)}$ time and outputs an estimator $\pE_{\tzeta}[\Theta]$ such that when $n_0 = \Omega\left( (d\log(d))^{\Omega(k)}/\gamma^2 \right)$ with probability $1-1/\poly(d)$ (over the draw of the input), 
\[
    \Norm{ (\Sigma^*)^{1/2}\Paren{\Theta^* - \pE_{\tzeta}[\Theta]} }_2  \leq \bigO{\lambda_k ~\epsilon^{1-1/k}  + \lambda_k  \gamma } \err_{\calD}(\Theta^*)^{1/2}
\]
Further, 
\[
\err_{\calD}\left(\pE_{\tzeta}[\Theta]\right) \leq \Paren{1 + \bigO{\lambda^2_k ~\epsilon^{2-2/k} + \lambda^2_k ~\gamma^2 } }   \err_{\calD}(\Theta^*).
\]
\end{theorem}

\paragraph{Efficient Estimator for Arbitrary Noise.}
We note that an argument similar to the one presented for Theorem \ref{thm:optimal_efficient_robust_regression} results in a polynomial time estimator when the regression instance does not have negatively correlated moments (definition \ref{def:negative_moments}), albeit at a slightly worse rate. Formally, 

\begin{corollary}[Robust Regression with Arbitrary Noise]
Consider the hypothesis of Theorem \ref{thm:optimal_efficient_robust_regression}, without the negatively correlated moments assumption. Then, there exists an algorithm that runs in time $n^{\mathcal{O}(k)}$ outputs an estimator $\tilde\Theta$
such that when $n_0 = (d\log(d))^{\Omega(k)}/\gamma^2$, with probability $1-1/\poly(d)$ (over the draw of the input),
\[
    \Norm{ (\Sigma^*)^{1/2}\Paren{\Theta^* - \tilde{\Theta}} }_2  \leq \bigO{\lambda_k ~\epsilon^{1-2/k} + c_2~\eta_2~\gamma} \err_{\calD}(\Theta^*)^{1/2}
\]
Further, 
\[
\err_{\calD}\left(\tilde{\Theta}\right) \leq \Paren{1 + \bigO{\lambda^2_k ~\epsilon^{2-4/k} + \lambda^2_2~\gamma} }   \err_{\calD}(\Theta^*)
\]
\end{corollary}

\begin{mdframed}
  \begin{algorithm}[Optimal Robust Regression in Polynomial Time]
    \label{algo:rounding-for-pseudo-distribution}\mbox{}
    \begin{description}
    \item[Input:] $n$ samples $\calX_{\epsilon}$ from the robust regression model $\mathcal{R}_{\calD}(\epsilon, \Theta^*, \Sigma^*)$.
    
    \item[Operation:]\mbox{}
    \begin{enumerate}
    \item Find a degree-$\bigO{k}$ pseudo-distribution $\tzeta$ satisfying $\calA_{\epsilon, \lambda_k}$ and minimizing $$\min_{w, x', y', \Theta} \pE_{\tzeta}\left[ \Paren{ \frac{1}{n} \sum_{i \in [n]} w_i\left( y'_i - \Iprod{\Theta,x'}\right)^2}^{k} \right]$$.
    \item Round the pseudo-distribution to obtain an estimator $ \pE_{\tzeta}[\Theta]$.
    \end{enumerate}
    \item[Output:] A vector $\pE_{\tzeta}[\Theta]$ such that the recovery guarantee in Theorem \ref{thm:optimal_efficient_robust_regression} is satisfied.
    \end{description}
  \end{algorithm}
\end{mdframed}

At a high level, we simply do not enforce the negatively correlated moments constraint in our polynomial system $\calA_{\epsilon,\lambda_k}$ and instead use the SoS Cauchy-Schwarz inequality in our key technical lemma (Lemma \ref{lem:key_sos_lemma}). For completeness, we provide the proof of the SoS lemma in Appendix \ref{sec:arbitrary_noise_polytime}. 

\subsection{Analysis}

We begin by observing that we can efficiently optimize the polynomial program above since it admits a compact representation. In particular, $\calA_{\epsilon,\lambda_k}$ can be represented as a system of $\poly(n^{k})$ constraints in $n^{O(k)}$ variables. We refer the reader to \cite{fleming2019semialgebraic} for a detailed overview on how to efficiently implement the aforementioned constraints.

\begin{lemma}[Soundness of the Constraint System]
\label{lem:soundness}
Given $n \geq n_0$ samples from $\calR_\calD(\epsilon, \Theta^*, \Sigma)$, with probability at least $1-1/\poly(d)$ over the draw of the samples, there exists an assignment for $w, x', y'$ and $\Theta$ such that $\calA_{\epsilon, \lambda_k}$ is feasible when $n_0 = \Paren{(d\log(d))^{\Omega(k)}}$. 
\end{lemma}
\begin{proof}
Consider the following assignment: for all $i\in [n]$ the $w_i$'s indicate the set of uncorrupted points in $\calX_\epsilon$, i.e. $w_i =1$ if $(x_i,y_i)=(x^*_i, y^*_i)$, $x_i' = x_i$ and $y_i' = y_i$. Further, $\Theta = \Theta^*$, the true hyperplane. It is easy to see that the first four constraints (intersection constraints) are satisfied. 

We observe that the marginal distribution over the covariates and the noise are both $(\lambda_k, k)$-certifiably hypercontractive since they are Affine transformations of $\calD$ (Fact \ref{fact:certifiable_hypercontractive_under_affine_transformations}). 
Next, it follows from Fact \ref{fact:sampling_preserves_hypercontract}, that for $n_0 = \Omega\left( d\log(d)^{\mathcal{O}(k)}\right)$, the uniform distribution over the samples $x_i$, is $(2 ~\lambda_k , k)$-certifiably hypercontractive with probability at least $1-1/\poly(d)$. Similarly, the uniform distribution on $y_i - \Iprod{ x_i,\Theta^*}$ is $(2 ~\lambda_k , k)$-certifiably hypercontractive. 


It remains to show that sampling preserves certifiable negatively correlated moments. Recall, since the joint distribution is hypercontractive, by Fact \ref{fact:sampling_preserves_hypercontract} we know that there's a degree $\bigO{k}$ proof of
\begin{equation}
\label{eqn:ncm_sampling}
\begin{split}
    \frac{1}{n} \sum_{i \in [n]} \Iprod{ v, x_i}^k \Paren{y_i - \Iprod{x_i,\Theta^*} }^k & \leq \bigO{\lambda_k^k} \Paren{ \frac{1}{n} \sum_{i \in [n]} \Iprod{ v, x_i}^2 \Paren{y_i - \Iprod{x_i,\Theta^*} }^2 }^{k/2}  \\
    & = \bigO{\lambda_k^k} \Paren{ \frac{1}{n} \sum_{i \in [n]} v^{\top} x_i (x_i)^\top  \Paren{y_i - \Iprod{x_i,\Theta^*} }^2 v }^{k/2}
\end{split}
\end{equation}
It thus suffices to bound the Operator norm of $\frac{1}{n}\sum_{i\in[n]}x_i x_i^\top  \Paren{y_i - \Iprod{x_i,\Theta^*} }^2$. It follows from Lemma \ref{lem:hypercontractive_lowner_sampling} that with probability at least $1-1/\poly(d)$, 
\begin{equation}
\label{eqn:lowner_upperbound}
    \frac{1}{n}\sum_{i\in[n]}x_i x_i^\top  \Paren{y_i - \Iprod{x_i, \Theta^*} }^2 \preceq \bigO{1} \expecf{x,y \sim\calD}{ xx^\top \Paren{ y - \Iprod{x,\Theta^*}}^2}
\end{equation}
when $n \geq n_0$. Using that $\calD$ has negatively correlated moments,
\begin{equation}
\label{eqn:population_ncm}
    \expecf{x,y \sim\calD}{ xx^\top \Paren{ y - \Iprod{x,\Theta^*}}^2} \preceq \expecf{x\sim \calD}{xx^\top} \expecf{x,y\sim\calD}{ \Paren{y- \Iprod{x, \Theta^*}}^2 } 
\end{equation}
Using Lemma \ref{lem:hypercontractive_lowner_sampling} on $xx^\top$ and $\Paren{y- \Iprod{x, \Theta^*}}^2$, we can bound \eqref{eqn:population_ncm} as follows: 
\begin{equation}
\label{eqn:lowner_hypercontractive}
    \expecf{x\sim \calD}{xx^\top} \expecf{x,y\sim\calD}{ \Paren{y- \Iprod{x, \Theta^*}}^2 }  \preceq  \bigO{1} \expecf{}{x_ix_i^\top} \Paren{y_i- \Iprod{x_i, \Theta^*}}^2
\end{equation}
Combining Equations \eqref{eqn:lowner_upperbound}, \eqref{eqn:population_ncm}, and \eqref{eqn:lowner_hypercontractive}, and substituting in \eqref{eqn:ncm_sampling}, we have
\begin{equation*}
    \frac{1}{n} \sum_{i \in [n]} \Iprod{ v, x_i}^k \Paren{y_i - \Iprod{x_i,\Theta^*} }^k \leq \bigO{\lambda_k^k} \Paren{ \frac{1}{n}\sum_{i \in [n]} \Iprod{x_i,v}^2 }^{\frac{k}{2}} \Paren{\frac{1}{n}\sum_{i \in[n]} \Paren{y_i- \Iprod{x_i, \Theta^*}}^2 }^{\frac{k}{2}}
\end{equation*}
which concludes the proof. 
\end{proof}

Let $\hat\Sigma$ be the empirical covariance of the uncorrupted samples $\calX$ and let $\hat\Theta$ be an optimizer for the empirical loss.
Applying Theorem \ref{thm:robust_identitifiability_independent} with $\calD$ being the uniform distribution on the uncorrupted samples $\calX$ and $\calD'$ being the uniform distribution on $x_i'$, we get 
\begin{equation*}
    \Norm{\hat{\Sigma}^{1/2}\Paren{\Theta - \hat{\Theta}}  }_2 \leq \bigO{ \lambda_k~ \epsilon^{1-1/k}} \err_{\calD}(\Theta^*)^{1/2}
\end{equation*}

Observe, the aforementioned bound is not a polynomial identity and thus cannot be expressed in the SoS framework. Therefore, we provide a low-degree SoS proof of a slightly modified version of the inequality above, that is inspired by our information theoretic identifiability proof in Theorem \ref{thm:robust_identitifiability_independent}. 

\begin{lemma}[Robust Identifiability in SoS]
\label{lem:key_sos_lemma}
Consider the hypothesis of Theorem \ref{thm:optimal_efficient_robust_regression}.
Let $w, x',y'$ and $\Theta$ be feasible solutions for the polynomial constraint system $\calA$. Let $\hat{\Theta} = \arg\min_{\Theta} \frac{1}{n}\sum_{i\in [n]} (y^*_i - \Iprod{x^*_i, \Theta} )^2$ be the empirical loss minimizer on the uncorrupted samples and let $\hat\Sigma = \expecf{}{x_i^* (x_i^*)^{\top}}$ be the covariance of the uncorrupted samples. Then, 
\begin{equation*}
\begin{split}
    \calA\sststile{4k}{w,x',y',\Theta} \Biggl\{  \Norm{ \hat\Sigma^{1/2}\Paren{\hat\Theta - \Theta} }^{2k}_2  &\leq 2^{3k}(2\epsilon)^{k-1} \lambda_k^k ~\sigma^{k/2}  \Norm{ \expecf{}{ x_i' (x_i')^{\top}}^{1/2}\Paren{\hat\Theta - \Theta} }^{k}_2 \\
    & \hspace{0.15in} + 2^{3k} (2\epsilon)^{k-2}\lambda_k^{2k} \Norm{ \hat\Sigma^{1/2} \Paren{\hat\Theta -\Theta}}^{2k}_2 \\
    & \hspace{0.15in}+ 2^{3k} (2\epsilon)^{k-1} \lambda_k^k    \expecf{}{\Paren{ y_i^* - \Iprod{x_i^*,\hat\Theta }}^2 }^{k/2} \Norm{ \hat\Sigma^{1/2}\Paren{\hat\Theta - \Theta} }^{k}_2 \Biggr\}
\end{split}
\end{equation*}
\end{lemma}
\begin{proof}
Consider the empirical covariance of the uncorrupted set given by $\hat\Sigma= \expecf{}{x_i^* (x_i^*)^{\top}}$. Then, using the \ref{eq:sos-substitution}, along with SoS Almost Triangle Inequality (Fact \ref{fact:sos-almost-triangle}),
\begin{equation}
\label{eqn:main_sos_bound}
\begin{split}
    \sststile{2k}{\Theta} \Biggl\{ \Iprod{v, \hat\Sigma \Paren{\hat\Theta - \Theta } }^k & = \Iprod{v, \expecf{}{x^*_i (x_i^*)^{\top}\Paren{\hat\Theta - \Theta } + x_i^*y_i^* - x_i^*y_i^*  }  }^k \\
    &= \Iprod{v, \expecf{}{x^*_i \Paren{ \Iprod{x_i^*,\hat\Theta }  - y_i^* } } + \expecf{}{x^*_i \Paren{ y_i^* - \Iprod{x_i^*,\Theta }  }}  }^k \\
    & \leq 2^k \Iprod{v, \expecf{}{x^*_i \Paren{ \Iprod{x_i^*,\hat\Theta }  - y_i^* } } }^k + 2^k\Iprod{v, \expecf{}{x^*_i \Paren{ y_i^* - \Iprod{x_i^*,\Theta }  }}  }^k  \Biggr\}
\end{split}
\end{equation}
Observe, the first term in \eqref{eqn:main_sos_bound} only consists of constants of the proof system. Since $\hat\Theta$ is the minimizer of $\expecf{}{\Paren{ \Iprod{x^*_i, \Theta } - y_i^*  }^2 }$, the gradient condition on the samples (appearing in Equation  \eqref{eqn:grad} of the indentifiability proof) implies this term is $0$. Therefore, applying the \ref{eq:sos-substitution} it suffices to bound the second term. 

To this end, we introduce the following auxiliary variables : for all $i\in [n]$, let $w'_i= w_i$ iff the $i$-th sample is uncorrupted in $\calX_\epsilon$, i.e. $x_i = x_i^*$. Then, it is easy to see that $\sum_i w'_i \geq (1-2\epsilon)n$. Further, since $\calA \sststile{2}{w} \Set{ (1- w_i' w_i)^2 = (1-w_i'w_i)}$, 
\begin{equation}
\label{eqn:bound_uncorrupted_not_indicated}
    \calA \sststile{2}{w} \Set{\frac{1}{n}\sum_{i \in [n]} (1-w'_iw_i)^2 =  \frac{1}{n}\sum_{i \in [n]} (1-w'_iw_i) \leq 2\epsilon }
\end{equation}
The above equation bounds the uncorrupted points in $\calX_\epsilon$ that are not indicated by $w$. Then, using the \ref{eq:sos-substitution}, along with the SoS Almost Triangle Inequality (Fact \ref{fact:sos-almost-triangle}),
\begin{equation}
\label{eqn:split_terms}
\begin{split}
    \calA \sststile{2k}{\Theta, w'} \Biggl\{\Iprod{v, \expecf{}{x^*_i \Paren{ y_i^* - \Iprod{x_i^*,\Theta }  }}  }^k & = \Iprod{v, \expecf{}{x^*_i \Paren{ y_i^* - \Iprod{x_i^*,\Theta }(w'_i + 1 -w'_i)  }}  }^k  \\
    & =  \Iprod{v, \expecf{}{w'_ix^*_i \Paren{ y_i^* - \Iprod{x_i^*,\Theta }}} + \expecf{}{(1-w'_i)x^*_i \Paren{ y_i^* - \Iprod{x_i^*,\Theta }}}  }^k \\
    &\leq 2^k\Iprod{v, \expecf{}{w'_ix^*_i \Paren{ y_i^* - \Iprod{x_i^*,\Theta }}}}^k  \\
    & \hspace{0.15in} + 2^k \Iprod{v, \expecf{}{(1-w'_i)x^*_i \Paren{ y_i^* - \Iprod{x_i^*,\Theta }}}}^k \Biggr\}
\end{split}
\end{equation}
Consider the first term of the last inequality in \eqref{eqn:split_terms}. Observe, since $w'_i x_i^*= w_i w_i'  x'_i$ and similarly, $w'_i y_i^*= w_i w_i'  y'_i$, 
\[\calA \sststile{4}{\Theta, w'} \Set{\expecf{}{w'_ix^*_i \Paren{ y_i^* - \Iprod{x_i^*,\Theta }}} = \expecf{}{w'_i w_i x'_i \Paren{ y_i' - \Iprod{x_i',\Theta }}}} 
\] 
For the sake of brevity, the subsequent statements hold for relevant SoS variables and have degree $O(k)$ proofs.  Using the \ref{eq:sos-substitution}, 

\begin{equation}
\label{eqn:grad_on_sos_var}
\begin{split}
    \calA \sststile{}{} \Biggl\{ \Iprod{v, \expecf{}{w'_ix^*_i \Paren{ y_i^* - \Iprod{x_i^*,\Theta }}}}^k & = \Iprod{v, \expecf{}{w'_i w_i x'_i \Paren{ y_i' - \Iprod{x_i',\Theta }}}}^k \\
    & =\Iprod{v, \expecf{}{  x'_i \Paren{ y_i' - \Iprod{x_i',\Theta }}} + \expecf{}{(1-w_i'w_i)  x'_i \Paren{ y_i' - \Iprod{x_i',\Theta }}}}^k \\
    & \leq 2^k \Iprod{v, \expecf{}{  x'_i \Paren{ y_i' - \Iprod{x_i',\Theta }}} }^k \\
    & \hspace{0.15in} + 2^k \Iprod{v,  \expecf{}{(1-w_i'w_i)  x'_i \Paren{ y_i' - \Iprod{x_i',\Theta }}} }^k  \Biggr\}
\end{split}
\end{equation}
Observe, the first term in the last inequality above is identically $0$, since we enforce the gradient condition on the SoS variables $x',y'$ and $\Theta$. We can then rewrite the second term using linearity of expectation, followed by applying SoS Hölder's Inequality (Fact \ref{fact:sos-holder}) combined with $\calA \sststile{2}{w} \Set{(1-w'_iw_i)^{2} = 1- w_i'w_i}$ to get 

\begin{equation}
\label{eqn:independence}
\begin{split}
    \calA \sststile{}{} \Biggl\{ \Iprod{v,  \expecf{}{(1-w_i'w_i) x'_i \Paren{ y_i' - \Iprod{x_i',\Theta }}} }^k & = \expecf{}{\Iprod{v,  (1-w_i') w_i x'_i \Paren{ y_i' - \Iprod{x_i',\Theta }} } }^k \\ 
    &  = \expecf{}{ (1-w_i'w_i) \Iprod{ v,  x'_i}   \Paren{ y_i' - \Iprod{ x_i',\Theta }}  }^k \\
    & \leq \expecf{}{(1-w_i'w_i)}^{k-1}\expecf{}{\Iprod{ v,  x'_i}^k   \Paren{ y_i' - \Iprod{ x_i',\Theta }}^k }  \\
    & \leq (2\epsilon)^{k-1} \expecf{}{\Iprod{ v, x'_i}^k   \Paren{ y_i' - \Iprod{ x_i',\Theta }}^k } \Biggr\}
\end{split}
\end{equation}
where the last inequality follows from Equation \eqref{eqn:bound_uncorrupted_not_indicated}. Next, we use the certifiable negatively correlated moments constraint with the \ref{eq:sos-substitution},
\begin{equation}
\label{eqn:negative_dependence}
    \calA \sststile{}{} \Set{\expecf{}{\Iprod{ v, x'_i}^k   \Paren{ y_i' - \Iprod{ x_i',\Theta }}^k }  \leq \bigO{\lambda^k_k} \expecf{}{\Iprod{ v,  x'_i}^2}^{\frac{k}{2}} \expecf{}{\Paren{ y_i' - \Iprod{ x_i',\Theta }}^2 }^{\frac{k}{2}} }
\end{equation}
For brevity, let $\sigma=  \expecf{}{\Paren{ y_i' - \Iprod{ x_i',\Theta }}^2 }$.  Using the \ref{eq:sos-substitution}, plugging Equation \eqref{eqn:negative_dependence} back into \eqref{eqn:independence}, we get 
\begin{equation}
\label{eqn:updated_independence}
    \begin{split}
    \calA \sststile{}{} \Biggl\{ \Iprod{v,  \expecf{}{(1-w_i')  x'_i \Paren{ y_i' - \Iprod{x_i',\Theta }}} }^k \leq (2\epsilon)^{k-1} \lambda_k^k  ~\sigma^{k/2}  \Iprod{ v,\expecf{}{  x'_i (x'_i)^{\top}}v }^{k/2}  \Biggr\}
\end{split}
\end{equation}
Recall, we have now bounded the first term of the last inequality in \eqref{eqn:split_terms}. 
Therefore, it remains to bound the second term of the last inequality in \eqref{eqn:split_terms}. Using the \ref{eq:sos-substitution}, we have 
\begin{equation}
\label{eqn:almost_tri_second}
\begin{split}
    \calA \sststile{}{} \Biggl\{\Iprod{v, \expecf{}{(1-w'_i)x^*_i \Paren{ y_i^* - \Iprod{x_i^*,\Theta }}}}^k & = \Iprod{v, \expecf{}{(1-w'_i)x^*_i \Paren{ y_i^* - \Iprod{x_i^*,\Theta -\hat\Theta +\hat\Theta }}}}^k\\
    &\leq 2^k \Iprod{v, \expecf{}{(1-w'_i)x^*_i \Paren{ y_i^* - \Iprod{x_i^*,\hat\Theta }}}}^k \\
    & \hspace{0.15in} + 2^k\Iprod{v, \expecf{}{(1-w'_i)x^*_i \Paren{  \Iprod{x_i^*,\Theta-\hat\Theta }}}}^k \Biggr\} 
\end{split}
\end{equation}
We again handle each term separately. Observe, the first term when decoupled is a statement about the uncorrupted samples. Therefore, using the SoS Hölder's Inequality (Fact \ref{fact:sos-holder}), 
\begin{equation}
\label{eqn:holder_to_first}
\begin{split}
\calA \sststile{}{} \Biggl\{ \Iprod{v, \expecf{}{(1-w'_i)x^*_i \Paren{ y_i^* - \Iprod{x_i^*,\hat\Theta }}}}^k &= \expecf{}{(1-w'_i)\Iprod{v, x^*_i \Paren{ y_i^* - \Iprod{x_i^*,\hat\Theta }}}}^k \\
& \leq \expecf{}{(1-w'_i)}^{k-1} \expecf{}{\Iprod{v, x^*_i \Paren{ y_i^* - \Iprod{x_i^*,\hat\Theta }}}^k} \\
& \leq (2\epsilon)^{k-1} \expecf{}{\Iprod{v, x_i^*}^k\Paren{ y_i^* - \Iprod{x_i^*,\hat\Theta }}^k }  \Biggr\}
\end{split}
\end{equation}
Observe, the uncorrupted samples have negatively correlated moments, and thus 
\[
\expecf{}{\Iprod{v, x_i^*}^k\Paren{ y_i^* - \Iprod{x_i^*,\hat\Theta }}^k } \leq \bigO{\lambda^k_k} \expecf{}{\Iprod{v, x_i^*}^2 }^{k/2} \expecf{}{\Paren{ y_i^* - \Iprod{x_i^*,\hat\Theta }}^2 }^{k/2} 
\]
Then, by the \ref{eq:sos-substitution}, we can bound \eqref{eqn:holder_to_first} as follows:

\begin{equation}
\label{eqn:upper_bound_uncorrupted}
    \calA \sststile{}{} \Biggl\{ \Iprod{v, \expecf{}{(1-w'_i)x^*_i \Paren{ y_i^* - \Iprod{x_i^*,\hat\Theta }}}}^k 
 \leq (2\epsilon)^{k-1} \lambda_k^k    \expecf{}{\Paren{ y_i^* - \Iprod{x_i^*,\hat\Theta }}^2 }^{k/2} \Iprod{v, \hat\Sigma v }^{k/2} \Biggr\}
\end{equation}
In order to bound the second term in \eqref{eqn:almost_tri_second}, we use the SoS Hölder's Inequality,
\begin{equation}
\label{eqn:holder_2/k}
\begin{split}
    \calA \sststile{}{} \Biggl\{ \Iprod{v, \expecf{}{(1-w'_i)x^*_i \Paren{  \Iprod{x_i^*,\Theta-\hat\Theta }}}}^k  & =  \expecf{}{ (1-w_i')^{k-2} \Iprod{v, x_i^* \Paren{\Iprod{x_i^*, \Theta-\hat\Theta }} }} \\
    & \leq \expecf{}{1-w'_i}^{k-2} \expecf{}{\Paren{v^{\top}x^*_i (x^*_i)^{\top} (\Theta-\hat\Theta) }^{\frac{k}{2}} }^2 \\
    & \leq (2\epsilon)^{k-2} \expecf{}{\Paren{v^{\top}x^*_i (x^*_i)^{\top} (\Theta-\hat\Theta) }^{\frac{k}{2}} }^2 \Biggr\}
\end{split}
\end{equation}
Combining the bounds obtained in \eqref{eqn:upper_bound_uncorrupted} and \eqref{eqn:holder_2/k}, we can restate Equation \eqref{eqn:almost_tri_second} as follows 
\begin{equation}
\label{eqn:second_term_in15}
\begin{split}
    \calA \sststile{}{} \Biggl\{\Iprod{v, \expecf{}{(1-w'_i)x^*_i \Paren{ y_i^* - \Iprod{x_i^*,\Theta }}}}^k &\leq 2^k (2\epsilon)^{k-1} \lambda_k^k    \expecf{}{\Paren{ y_i^* - \Iprod{x_i^*,\hat\Theta }}^2 }^{k/2} \Iprod{v, \hat\Sigma v }^{k/2} \\
    & \hspace{0.15in} + 2^k(2\epsilon)^{k-2} \expecf{}{\Paren{v^{\top}x^*_i (x^*_i)^{\top} (\Theta-\hat\Theta) }^{k/2} }^2 \hspace{0.2in} \Biggr\} 
\end{split}
\end{equation}
Combining \eqref{eqn:second_term_in15} with \eqref{eqn:updated_independence}, we obtain an upper bound for the last inequality in Equation \eqref{eqn:split_terms}. Therefore, using the \ref{eq:sos-substitution}, we obtain 
\begin{equation}
\begin{split}
    \calA \sststile{}{} \Biggl\{\Iprod{v, \expecf{}{x^*_i \Paren{ y_i^* - \Iprod{x_i^*,\Theta }  }}  }^k 
    &\leq 2^k(2\epsilon)^{k-1} \lambda_k^k  ~\sigma^{k/2}  \Iprod{ v,\expecf{}{  x'_i (x'_i)^{\top}}v }^{k/2} \\
    & \hspace{0.15in} + 2^{2k} (2\epsilon)^{k-2} \expecf{}{\Paren{v^{\top}x^*_i (x^*_i)^{\top} (\Theta-\hat\Theta) }^{\frac{k}{2}} }^2 \\
    & \hspace{0.15in}+ 2^{2k} (2\epsilon)^{k-1} \lambda_k^k  \expecf{}{\Paren{ y_i^* - \Iprod{x_i^*,\hat\Theta }}^2 }^{k/2} \Iprod{v, \hat\Sigma v }^{k/2} \Biggr\}
\end{split}
\end{equation}
Recall, an upper bound on Equation \eqref{eqn:main_sos_bound} suffices to obtain an upper bound on $\Iprod{v, \hat\Sigma\Paren{\hat\Theta - \Theta} }$ as follows:

\begin{equation}
\label{eqn:final_bound_v}
\begin{split}
    \calA \sststile{}{} \Biggl\{ \Iprod{v, \hat\Sigma \Paren{\hat\Theta - \Theta } }^k  &\leq 2^{2k}(2\epsilon)^{k-1} \lambda_k^k  ~\sigma^{k/2}  \Iprod{ v,\expecf{}{  x'_i (x'_i)^{\top}}v }^{k/2} \\
    & \hspace{0.15in} + 2^{3k} (2\epsilon)^{k-2} \expecf{}{\Paren{v^{\top}x^*_i (x^*_i)^{\top} (\Theta-\hat\Theta) }^{\frac{k}{2}} }^2 \\
    & \hspace{0.15in}+ 2^{3k} (2\epsilon)^{k-1} \lambda_k^k   \expecf{}{\Paren{ y_i^* - \Iprod{x_i^*,\hat\Theta }}^2 }^{k/2} \Iprod{v, \hat\Sigma v }^{k/2} \Biggr\}
\end{split}
\end{equation}
Consider the substitution $v \mapsto \Paren{\hat\Theta - \Theta}$. Then, 
\begin{equation*}
    \begin{split}
        \Iprod{v, \hat\Sigma \Paren{\hat\Theta - \Theta } }^k &= \Norm{ \hat\Sigma^{1/2}\Paren{\hat\Theta - \Theta} }^{2k}_2 \\
        \Iprod{ v,\expecf{}{  x'_i (x'_i)^{\top}}v }^{k/2}&= \Norm{ \expecf{}{ x_i' (x_i')^{\top}}^{1/2}\Paren{\hat\Theta - \Theta} }^{k}_2 \\
        \expecf{}{\Paren{v^{\top}x^*_i (x^*_i)^{\top} (\Theta-\hat\Theta) }^{\frac{k}{2}} }^2& = \expecf{}{ \langle x_i^* , \hat\Theta-\Theta \rangle^{k}}^{2} \leq \lambda_k^{2k} \Norm{ \hat\Sigma^{1/2} \Paren{\hat\Theta -\Theta}}^{2k}_2 \\
        \Iprod{v, \hat\Sigma v}^{k/2} & = \Norm{ \hat\Sigma^{1/2}\Paren{\hat\Theta - \Theta} }^{k}_2
    \end{split}
\end{equation*}
Combining the above with \eqref{eqn:final_bound_v}, we conclude 
\begin{equation}
    \begin{split}
    \calA \sststile{}{} \Biggl\{  \Norm{ \hat\Sigma^{1/2}\Paren{\hat\Theta - \Theta} }^{2k}_2  &\leq 2^{3k}(2\epsilon)^{k-1} \lambda_k^k ~\sigma^{k/2}  \Norm{ \expecf{}{ x_i' (x_i')^{\top}}^{1/2}\Paren{\hat\Theta - \Theta} }^{k}_2 \\
    & \hspace{0.15in} + 2^{3k} (2\epsilon)^{k-2}\lambda_k^{2k} \Norm{ \hat\Sigma^{1/2} \Paren{\hat\Theta -\Theta}}^{2k}_2 \\
    & \hspace{0.15in}+ 2^{3k} (2\epsilon)^{k-1} \lambda_k^k    \expecf{}{\Paren{ y_i^* - \Iprod{x_i^*,\hat\Theta }}^2 }^{k/2} \Norm{ \hat\Sigma^{1/2}\Paren{\hat\Theta - \Theta} }^{k}_2 \Biggr\}
\end{split}
\end{equation}

\end{proof}

Next, we relate the covariance of the samples indicated by $w$ to the covariance on the uncorrupted points. Observe, a real world proof of this follows simply from Fact \ref{fact:hypercontractive_covariance}. 
\begin{lemma}[Bounding Sample Covariance]
\label{lem:robust_sample_covariance}
Consider the hypothesis of Theorem \ref{thm:optimal_efficient_robust_regression}. Let $w, x',y'$ and $\Theta$ be feasible solutions for the polynomial constraint system $\calA$. Then, for $\delta \leq \bigO{\lambda_k  \epsilon^{1-1/k}}<1$, 
\begin{equation*}
    \calA \sststile{2k}{w,x'} \Set{\Iprod{v,\expecf{}{ x'_i (x_i')^{\top}} v}^{k/2} \leq \Paren{1+\bigO{\delta^{k/2}}} \Iprod{v,\hat\Sigma v}^{k/2} } 
\end{equation*}
\end{lemma}
\begin{proof}
Our proof closely follows Lemma 4.5 in \cite{kothari2017outlier}. 
For $i \in [n]$, let $z_i$ be an indicator variable such $z_i (x^*_i - x'_i) =0$. Observe, there exists an assignment to $z_i$ such that $\sum_{i\in [n]} z_i  = (1-\epsilon)n$, since at most $\epsilon n$ points were corrupted. Further, $z^2_i = z_i$ and $\frac{1}{n} z_i = \epsilon$. Then, using the \ref{eq:sos-substitution},
\begin{equation}
\label{eqn:cov_identifiability}
\begin{split}
    \calA \sststile{2k}{w,x'} \Biggl\{ \Iprod{v, \Paren{\expecf{}{ x'_i (x_i')^{\top}} -\hat\Sigma} v}^k & = \Iprod{v, \expecf{}{ (1 + z_i -z_i) \Paren{x'_i (x_i')^{\top} - x^*_i (x_i^*)^{\top}} } v}^k   \\ 
    & =  \expecf{}{ (1  -z_i) \Iprod{v, \Paren{x'_i (x_i')^{\top} - x^*_i (x_i^*)^{\top}}, v }}^{k} \\
    & \leq \epsilon^{k-2} \cdot  \expecf{}{ \Paren{\Iprod{v, x'_i}^2 - \Iprod{v, x^*_i}^2}^{k/2}}^2\\
    & \leq \epsilon^{k-2} \expecf{}{ 2^{k/2}\Iprod{v, x'_i}^k + 2^{k/2}\Iprod{v, x^*_i}^k }^2  \\
    & \leq \epsilon^{k-2} 2^k \Paren{ c^k_k~ \expecf{}{ \Iprod{v, x'_i}^2 }^{k/2} + \lambda_k^k ~\expecf{}{ \Iprod{v, x^*_i}^2 }^{k/2} }^2  \hspace{0.1in} \Biggr \}
\end{split}
\end{equation}
where the first inequality follows from applying the SoS Hölder's Inequality, the second follows from the SoS Almost Triangle Inequality and the third inequality follows from certifiable hypercontractivity of the SoS variables and the uncorrupted  samples. Using the SoS Almost Triangle Inequality again, we have 
\begin{equation}
\label{eqn:cov_almost_triangle}
   \calA \sststile{}{} \Set{\Paren{ c^k_k~ \expecf{}{ \Iprod{v, x'_i}^2 }^{k/2} + \lambda_k^k ~\expecf{}{ \Iprod{v, x^*_i}^2 }^{k/2} }^2 \leq \lambda_k^{2k} ~2^2~ \Paren{ \Iprod{v, \expecf{}{x'_i (x'_i)^\top  v}}^k + \Iprod{v, \hat\Sigma v}^k }    }  
\end{equation}
Combining Equations \ref{eqn:cov_identifiability}, \ref{eqn:cov_almost_triangle}, we obtain 
\begin{equation}
    \calA  \sststile{}{} \Biggl\{ \Iprod{v, \Paren{\expecf{}{ x'_i (x_i')^{\top}} -\hat\Sigma} v}^k  \leq \epsilon^{k-2} ~\lambda_k^{2k} ~2^{k+2} \Iprod{v, \Paren{\expecf{}{x'_i (x'_i)^\top }+ \hat\Sigma}  v}^k  \Biggr\}
\end{equation}
Using Lemma A.4 from \cite{kothari2017outlier}, rearranging and setting $k = k/2$ yields the claim. 
\end{proof}

\begin{lemma}[Rounding]
\label{lem:rounding}
Consider the hypothesis of Theorem \ref{thm:optimal_efficient_robust_regression}. Let $\hat{\Theta} = \arg\min_{\Theta} \frac{1}{n}\sum_{i\in [n]} (y^*_i - \Iprod{x^*_i, \Theta} )^2$ be the empirical loss minimizer on the uncorrupted samples.
Then, 
\begin{equation*}
    \Norm{\hat\Sigma^{1/2}\Paren{\hat\Theta- \pE_{\tzeta}[\Theta] }  }_2 \leq  \bigO{\epsilon^{1-\frac{1}{k}}~\lambda_k } \Paren{ \pE_{\tzeta}\left[ \expecf{}{\Paren{ y_i'
    - \Iprod{x_i',\Theta }}^2 }^{k} \right]^{\frac{1}{2k}} + \expecf{}{  \Paren{ y_i^* - \Iprod{x_i^*,\hat\Theta }}^2 }^{\frac{1}{2}} }
\end{equation*}
\end{lemma}
\begin{proof}
Observe, combining Lemma \ref{lem:key_sos_lemma} and Lemma \ref{lem:robust_sample_covariance}, we obtain 
\begin{equation}
\label{eqn:final_sos_bound}
\begin{split}
    \calA\sststile{}{} \Biggl\{  \Norm{ \hat\Sigma^{1/2}\Paren{\hat\Theta - \Theta} }^{2k}_2 & \leq \bigO{ \frac{2^{3k}\epsilon^{k-1} ~\lambda_k^k  }{1+2^{3k} (2\epsilon)^{k-2}\lambda_k^{2k}} }\Norm{ \hat\Sigma^{1/2}\Paren{\hat\Theta - \Theta} }^{k}_2  \\ 
    & \hspace{0.2in} \Paren{ \expecf{}{\Paren{ y_i'
    - \Iprod{x_i',\Theta }}^2 }^{\frac{k}{2}} + \expecf{}{\Paren{ y_i^* - \Iprod{x_i^*,\hat\Theta }}^2 }^{\frac{k}{2}} } \Biggr\}
\end{split}
\end{equation}
Using Cancellation within SoS (Fact \ref{lem:cancellation-sos}) along with the SoS Almost Triangle Inequality, we can conclude 
\begin{equation}
\label{eqn:final_sos_bound_cancellation}
\begin{split}
    \calA\sststile{}{} \Biggl\{  \Norm{ \hat\Sigma^{1/2}\Paren{\hat\Theta - \Theta} }^{2k}_2 & \leq \bigO{2^{3k}~\epsilon^{k-1} ~\lambda_k^k  }^{2} \\ 
    & \hspace{0.2in} \Paren{ \expecf{}{\Paren{ y_i'
    - \Iprod{x_i',\Theta }}^2 }^{k} + \expecf{}{\Paren{ y_i^* - \Iprod{x_i^*,\hat\Theta }}^2 }^{k} } \Biggr\}
\end{split}
\end{equation}

Recall, $\tzeta$ is a degree-$\bigO{k}$ pseudo-expectation satisfying $\calA$. Therefore, it follows from Fact \ref{fact:sos-soundness} along with Equation \ref{eqn:final_sos_bound},
\begin{equation}
\label{eqn:taking_pseudoexp}
\begin{split}
    \pE_{\tzeta }\left[\Norm{ \hat\Sigma^{\frac{1}{2}}\Paren{\hat\Theta - \Theta} }^{2k}_2\right]
    & \leq \bigO{2^{4k}~\epsilon^{k-1} ~\lambda_k^k }^{2}  \\
    &\hspace{0.2in} 
     \Paren{  \pE_{\tzeta}\left[\expecf{}{\Paren{ y_i'
    - \Iprod{x_i',\Theta }}^2 }^{k} \right]  +  \expecf{}{\Paren{ y_i^* - \Iprod{x_i^*,\hat\Theta }}^2 }^{k}  }
\end{split}
\end{equation}
Further, using Fact \ref{fact:pseudo-expectation-holder}, we have  $ \Norm{ \hat\Sigma^{\frac{1}{2}}\Paren{\hat\Theta - \pE_{\tzeta} \Theta} }^{2k}_2 \leq \pE_{\tzeta }\left[\Norm{ \hat\Sigma^{\frac{1}{2}}\Paren{\hat\Theta - \Theta} }^{2k}_2\right]$. Substituting above and taking the $(1/2k)$-th root, 
\begin{equation}
\begin{split}
    \Norm{ \hat\Sigma^{\frac{1}{2}}\Paren{\hat\Theta - \pE_{\tzeta} [\Theta]} }_2 & \leq \bigO{\epsilon^{1-\frac{1}{k}}~\lambda_k  } \Paren{  \pE_{\tzeta}\left[\expecf{}{\Paren{ y_i'
    - \Iprod{x_i',\Theta }}^2 }^{k} \right]  +  \expecf{}{\Paren{ y_i^* - \Iprod{x_i^*,\hat\Theta }}^2 }^{k}  }^{1/2k} \\
    & \leq \bigO{\epsilon^{1-\frac{1}{k}}~\lambda_k } \Paren{ \pE_{\tzeta}\left[ \expecf{}{\Paren{ y_i'
    - \Iprod{x_i',\Theta }}^2 }^{k} \right]^{\frac{1}{2k}} + \expecf{}{  \Paren{ y_i^* - \Iprod{x_i^*,\hat\Theta }}^2 }^{\frac{1}{2}} }
\end{split}
\end{equation}
which concludes the proof.  
\end{proof}

\begin{lemma}[Bounding Optimization and Generalization Error]
\label{lem:bounding_opt_and_gen_error}
Under the hypothesis of Theorem \ref{thm:optimal_efficient_robust_regression}, 
\begin{enumerate}
    \item $\pE_{\tzeta}\left[ \expecf{}{\Paren{ y_i'
    - \Iprod{x_i',\Theta }}^2 }^{k} \right]^{\frac{1}{2k}} \leq  \expecf{}{ y_i^* - \Iprod{x_i^*, \hat\Theta}^2 }^{\frac{1}{2}} $, and
    \item For any $\zeta >0$, if $n \geq n_0$, such that $n_0 = \Omega\left(\max\{c_4 d/\zeta^2, d^{\mathcal{O}(k)}\}\right)$, with probability at least $1-1/\poly(d)$,  $\expecf{}{ y_i^* - \Iprod{x_i^*, \hat\Theta}^2 }^{\frac{1}{2}} \leq \Paren{1+ \zeta} \expecf{x,y \sim \calD}{  y - \Iprod{x, \Theta^*}^2 }^{\frac{1}{2}}$.
\end{enumerate}
\end{lemma}
\begin{proof}
We exhibit a degree-$\bigO{k}$ pseudo-distribution $\hat\zeta$ such that it is supported on a point mass and attains objective value at most $\expecf{}{ y_i^* - \Iprod{x_i^*, \hat\Theta}^2 }^{\frac{1}{2}}$. Since our objective function minimizes over all degree-$\bigO{k}$ pseudo-distributions, the resulting objective value w.r.t. $\tzeta$ can only be better. Let $\hat\zeta$ be the pseudo-distribution supported on $(w, x^*, y^*, \hat\Theta)$ such that $w_i =1$ if $x_i = x_i^*$ (i.e. the $i$-th sample is not corrupted.) It follows from $n \geq n_0$ and  Lemma \ref{lem:soundness} that this assignment satisfies the constraint system $\calA_{\epsilon, \lambda_k}$. Then, the objective value satisfies
\begin{equation}
    \pE_{\tzeta}\left[ \expecf{}{\Paren{ y_i'
    - \Iprod{x_i',\Theta }}^2 }^{k} \right] \leq \pE_{\hat\zeta}\left[ \expecf{}{\Paren{ y_i'
    - \Iprod{x_i',\Theta }}^2 }^{k}  \right] = \expecf{}{ \Paren{y_i^* - \Iprod{x_i^*, \hat\Theta }}^2}^k
\end{equation}
Taking $(1/2k)$-th roots yields the first claim. 

To bound the second claim, let $\calU$ be the uniform distribution on the uncorrupted samples, $x_i^*, y_i^*$. Observe, by optimality of $\hat\Theta$ on the uncorrupted samples, $\err_{\calU}(\hat\Theta) \leq \err_{\calU}(\Theta^*)$. Consider the random variable $z_i = \Paren{y_i^* - \Iprod{x_i^*, \Theta^*} }^2 - \expecf{x,y\sim\calD}{ \Paren{y - \Iprod{x, \Theta^*}}^2 }$. Since $\expecf{}{z_i} = 0$, we apply Chebyschev's inequality to obtain
\begin{equation*}
\begin{split}
    \Pr\left[ \frac{1}{n} \sum_{i\in [n]} z_i \geq  \zeta \right]   = \frac{\expecf{}{ z_1^2}}{\zeta^2 n}
    & \leq \frac{ \expecf{}{ \Paren{y - \Iprod{x, \Theta }}^4 } }{\zeta^2 n} \\
    & \leq c_4 \frac{\err_{\calD}(\Theta^*)^2 }{n \zeta^2 }
\end{split}
\end{equation*}
Therefore, with probability at least $1-\delta$, 
\begin{equation*}
    \err_{\calU}(\hat\Theta) \leq  \Paren{1 + \sqrt{\frac{c_4 }{n \delta}} }\err_{\calD}(\Theta^*)
\end{equation*}
Therefore, setting $n = \Omega(c_4 d /\zeta^2)$, it follows that with probability $1-1/\poly(d)$, for any $\zeta > 0$, 
\begin{equation*}
    \err_{\calU}(\hat\Theta) \leq  \Paren{1 + \zeta }\err_{\calD}(\Theta^*)
\end{equation*}
Taking square-roots concludes the proof. 
\end{proof}

\begin{proof}[Proof of Theorem \ref{thm:optimal_efficient_robust_regression}]

Given $n \geq n_0$ samples, it follows from Lemma \ref{lem:soundness}, that with probability $1-1/\poly(d)$, the constraint system $\calA_{\epsilon, \lambda_k}$ is feasible. Let $\xi_1$ be the event that the system is feasible and condition on it. Then, it follows from Lemma \ref{lem:rounding} and Lemma \ref{lem:bounding_opt_and_gen_error}, with probability $1-1/\poly(d)$, 
\begin{equation}
\label{eqn:sos_loss_upperbound}
    \Norm{\hat\Sigma^{1/2} \Paren{ \pE_{\tzeta}[\Theta]  - \hat\Theta } }_2 \leq \bigO{\lambda_k~\epsilon^{1-1/k} } \err_{\calD}(\Theta^*)^{1/2}
\end{equation}
Let $\xi_2$ be the event that \eqref{eqn:sos_loss_upperbound} holds and condition on it. It then follows from Fact \ref{fact:convergnce_of_empirical_moments}, with probability $1-1/\poly(d)$, 
\begin{equation}
\label{eqn:sos_loss_lowerbound}
    \Norm{ \Paren{\Sigma^* }^{1/2}\Paren{  \pE_{\tzeta}[\Theta]  - \hat\Theta } }_2 \leq \bigO{\lambda_k~\epsilon^{1-1/k} } \err_{\calD}(\Theta^*)^{1/2}
\end{equation}
Let $\xi_2$ be the event that \eqref{eqn:sos_loss_lowerbound} holds and condition on it. It remains to relate the hyperplanes $\hat\Theta$ and $\Theta^*$. By reverse triangle inequality, 

\begin{equation*}
    \Norm{ \Paren{\Sigma^* }^{1/2}\Paren{  \pE_{\tzeta}[\Theta]  - \Theta^* } }_2 - \Norm{ \Paren{\Sigma^* }^{1/2}\Paren{ \Theta^*  - \hat\Theta } }_2 \leq \Norm{ \Paren{\Sigma^* }^{1/2}\Paren{  \pE_{\tzeta}[\Theta]  - \hat\Theta } }_2
\end{equation*}
Using normal equations, we have $\hat\Theta = \hat\Sigma^{-1} \expecf{}{x_i y_i}$ and $\Theta^* = (\Sigma^*)^{-1}\expecf{}{xy}$. Since $\hat\Sigma \preceq (1+0.01) \Sigma^*$, 

\begin{equation}
\begin{split}
    \Norm{ \Paren{\Sigma^* }^{1/2}\Paren{ \Theta^*  - \hat\Theta } }_2 & = \Norm{ \Paren{\Sigma^* }^{1/2}\Paren{ \hat\Sigma^{-1}\hat\Sigma \Theta^*  - \hat\Sigma^{-1} \expecf{}{x_i y_i} } }_2 \\
    & = \Norm{ \Paren{\Sigma^* }^{1/2} \hat\Sigma^{-1}\Paren{  \expecf{}{x_i \Paren{y_i - x_i^\top \Theta^*}} } }_2 \\
    & \leq 1.01 \Norm{   \expecf{}{\Paren{\Sigma^* }^{-1/2} x_i \Paren{y_i - x_i^\top \Theta^*}}  }_2
\end{split}
\end{equation}
By Jensen's inequality
\begin{equation*}
    \expecf{}{\Norm{\frac{1}{n} \sum_{i \in [n]}{\Paren{\Sigma^* }^{-1/2} x_i \Paren{y_i - x_i^\top \Theta^*}}  }_2} \leq \sqrt{ \expecf{}{\Norm{\frac{1}{n} \sum_{i \in [n]}{\Paren{\Sigma^* }^{-1/2} x_i \Paren{y_i - x_i^\top \Theta^*}}  }^2_2}}
\end{equation*}
Let $z_i = \sum_{i \in [n]}{\Paren{\Sigma^* }^{-1/2} x_i \Paren{y_i - x_i^\top \Theta^*}}$. Let $(\sum_{i\in [n]} z_i)_1$ denote the first coordinate of the vector. We bound the expectation of this coordinate as follows: 

\begin{equation}
\label{eqn:coordinate_wise_bound}
\begin{split}
    \expecf{}{(\sum_{i\in [n]} z_i)^2_1} & = \frac{1}{n^2} \expecf{}{\sum_{i,i'\in [n]} \Paren{(\Sigma^*)^{-1} x_i x_{i'}}_1 \Paren{y_i - x_i^{\top} \Theta^*} \Paren{y_{i'} - x_{i'}^{\top} \Theta^*}   } \\ 
    & = \frac{1}{n^2} \expecf{}{ \sum_{i\in[n]} \Paren{ (\Sigma^*)^{-1}x_i^2}_1 \Paren{y_{i} - x_{i}^{\top} \Theta^*}^2  }\\
    & =\frac{1}{n} \expecf{}{ (\Sigma^*)^{-1} (x)^2_1 \Paren{y - x^\top \Theta^*}}
\end{split}
\end{equation}
where the second equality follows from independence of the samples. Using negatively correlated moments, we have

\begin{equation*}
     \expecf{}{(\Sigma^*)^{-1} (x)^2_1 \Paren{y - x^\top \Theta^*}^2 } \leq  \expecf{}{ (\Sigma^*)^{-1} (x)^2_1} \expecf{}{ \Paren{y - x^\top \Theta^*}^2} 
\end{equation*}
Setting $v = (\Sigma^*)^{1/2}e_1$ and using Hypercontractivity of the covariates and the noise in the above equation, 
\begin{equation}
\label{eqn:4_hypercontractivity}
     \expecf{}{ \Sigma^{-1} (x)^2_1} \expecf{}{ \Paren{y - x^\top \Theta^*}^2}  \leq  \bigO{c_2^2 ~\eta^2_2} \err_{\calD}(\Theta^*)
\end{equation}
Summing over the coordinates, and combining \eqref{eqn:coordinate_wise_bound}, \eqref{eqn:4_hypercontractivity},  we obtain 

\begin{equation}
    \expecf{}{\Norm{\frac{1}{n} \sum_{i \in [n]}{\Paren{\Sigma^* }^{-1/2} x_i \Paren{y_i - x_i^\top \Theta^*}}  }_2} \leq  \bigO{c_2 \eta_2} \sqrt{\frac{d ~\err_{\calD}(\Theta^*) }{n}}
\end{equation}
Applying Chebyschev's Inequality , with probability $1-\delta$

\begin{equation*}
    \Norm{ (\Sigma^*)^{1/2} \Paren{\Theta^* - \E_{\tzeta}[\Theta]} }_2 \leq \bigO{\lambda_k ~\epsilon^{1-1/k} + c_2 ~\eta_2 \sqrt{\frac{d}{\delta n}} } \err_{\calD}(\Theta^*)^{1/2}
\end{equation*} 
Since $n \geq n_0$, we can simplify the above bound and obtain the claim. 

The running time of our algorithm is clearly dominated by computing a degree-$\bigO{k}$ pseudo-distribution satisfying $\calA_{\epsilon, \lambda_k}$. Given that our constraint system consists of $O(n)$ variables and $\poly(n)$ constraints, it follows from Fact \ref{fact:eff-pseudo-distribution} that the pseudo-distribution $\tzeta$ can be computed in $n^{\mathcal{O}(k)}$ time.  

\end{proof}

\section{Lower bounds}

In this section, we present information-theoretic lower bounds on the rate of convergence of parameter estimation and least-squares error for robust regression. Our constructions proceed by demonstrating two distributions over regression instances that are $\epsilon$-close in total variation distance and the marginal distribution over the covariates is hypercontractive, yet the true hyperplanes are $f(\epsilon)$-far in scaled $\ell_2$ distance. 

\subsection{True Linear Model}
Consider the setting where there exists an optimal hyperplane $\Theta^*$ that is used to generate the data, with the addition of independent noise added to each sample, i.e. 
\[ y = \inprod{x}{\Theta^*} + \omega, \]
where $\omega$ is independent of $x$. Further, we assume that covariates and noise are hypercontractive.  In this setting, Theorem \ref{thm:robust_identitifiability_independent} implies that we can recover a hyperplane close to $\Theta^*$ at a rate proportional to $\epsilon^{1-1/k}$. We show that this dependence is tight for $k=4$. We note that independent noise is a special case of the distribution having negatively correlated moments. 

\begin{theorem}[True Linear Model Lower Bound, Theorem \ref{thm:informal_independent_noise_lowerbound} restated]
\label{thm:lower_bound_independent_noise}
For any $\epsilon >0$, there exist two distributions $\calD_1, \calD_2$ over $\R^2 \times\R $ such that the marginal distribution over $\R^2$ has covariance $\Sigma$ and is $(c_k,k)$-hypercontractive yet $\Norm{\Sigma^{1/2}(\Theta_{1} - \Theta_{2} ) }_2 =\bigOmega{ \sqrt{c_k}~\sigma~\epsilon^{1-1/k}} $, where $\Theta_1 , \Theta_2$ be the optimal hyperplanes for $\calD_1 $ and $\calD_2$ respectively, $\sigma = \max( \err_{\calD_1}(\Theta_1),  \err_{\calD_2}(\Theta_2)) < 1/\epsilon^{1/k}$ and the noise $\omega$ is uniform over $[-\sigma, \sigma]$.
\end{theorem}



\begin{proof}
We construct a $2$-dimensional instance where the marginal distribution over covariates is identical for $\calD_1$ and $\calD_2$. The pdf is given as follows: for $q\in \{1,2\}$ on the first coordinate,  $x_1$,
\begin{equation*}
\begin{split}
\calD_q(x_1) = \begin{cases}
1/2 , ~~&  \text{if} ~~ x_1 \in [-1,1] \\
0  & \text{otherwise} ~~
\end{cases}
\end{split}
\end{equation*}
and on the second coordinate, $x_2$, 
\begin{equation*}
\begin{split}
\calD_q(x_2) = \begin{cases}
\epsilon/2, ~~&  \text{if} ~~ x_2 \in \{ -1/\epsilon^{1/k}, 1/\epsilon^{1/k} \} \\
\frac{1-\epsilon}{2\epsilon\sigma }  & \text{if} ~~ x_2 \in [-\epsilon\sigma , \epsilon\sigma ]\\
0 & \text{otherwise}
\end{cases}
\end{split}
\end{equation*}
Next, we set $\Theta_1 = (1,1)$, $\Theta_2 = (1,-1)$ and $\omega$ to be uniform over $[-\sigma, \sigma]$. Therefore, 
\begin{equation}
\label{eqn:conditional_dist}
\begin{split}
\calD_1(y \mid (x_1,x_2) ) &= x_1 + x_2 + \omega ~~~~\textrm{ and}\\
\calD_2(y \mid (x_1,x_2) ) &= x_1 - x_2 + \omega
\end{split}
\end{equation}
Observe, $\expecf{}{x^k_1} = \int^{1}_{-1} x^k/2 = 1/(k+1)$ and $\expecf{}{x^2_1} =\int^{1}_{-1} x^2/2 = 1/3$. Further, 
\begin{equation*}
\begin{split}
     \Exp[x_2^k] & = \frac{(1-\epsilon)}{\epsilon\sigma } \frac{(\epsilon\sigma)^{k+1}}{k+1} + \epsilon \cdot \Paren{ \frac{1}{\epsilon^{1/k}}}^{k} = 1 + \frac{(1-\epsilon)}{(k+1)} (\epsilon\sigma)^k \\
     \Exp[x_2^2] & = \frac{(1-\epsilon)}{3\epsilon\sigma } (\epsilon\sigma)^3 + \epsilon \cdot \Paren{ \frac{1}{\epsilon^{1/k}} }^2 = \epsilon^{1-2/k} + \frac{1-\epsilon}{3}(\epsilon \sigma)^2
\end{split}
\end{equation*}
Observe,  $\Exp[x_2^k] \leq  (1/(c\epsilon^{k/2 - 1})) ~ \Exp[x_2^2]^{k/2} $, for a fixed constant $c$. Then, for any unit vector $v$, 
\begin{equation*}
\begin{split}
    \expecf{}{\Iprod{x, v}^k} \leq \expecf{}{ (2 x_1 v_1)^k + (2 x_2 v_2)^k } & \leq  c^{k/2}_k~ \Paren{ \expecf{}{(x_1 v)^2}^{k/2}  + \expecf{}{(x_2 v)^2}^{k/2} }\\
    & \leq c^{k/2}_k~ \expecf{}{\Iprod{x, v}^2}^{k/2}
\end{split}
\end{equation*}
where $c^{k/2}_k = 2^k /c\epsilon^{k/2 - 1}$. Therefore, $\calD_1,\calD_2$ are $(c_k, k)$-hypercontractive over $\R^2$. Next, we compute the TV distance between the two distributions. 
\begin{equation}
\label{eqn:tv_computation}
    \begin{split}
        d_{\tv}\Paren{\calD_1, \calD_2} &= \frac{1}{2}\int_{\R^2 \times \R} \abs{\calD_1(x_1, x_2, y) - \calD_2(x_1, x_2, y) } \\
        & =  \frac{1}{2}\int_{\R^2} \calD_1(x_1, x_2) \int_{\R}\abs{\calD_1(y \mid (x_1, x_2) ) - \calD_2( y\mid (x_1, x_2) ) }
    \end{split}
\end{equation}
where the last equality follows from the definition of conditional probability. It follows from Equation \eqref{eqn:conditional_dist} that $\calD_1(y \mid (x_1, x_2) ) = \calU(x_1+x_2 -\sigma, x_1 +x_2 +\sigma)$ and $\calD_2( y\mid (x_1, x_2) ) = \calU(x_1-x_2 -\sigma, x_1 -x_2 +\sigma)$.  If $\abs{x_2} \geq \sigma$ the intervals are disjoint and $\abs{\calD_1(y \mid (x_1, x_2) ) - \calD_2( y\mid (x_1, x_2) ) } = 2$. If $\abs{x_2} < \sigma$, then two symmetric non-intersecting regions have mass $ 2\abs{x_2}/2\sigma$ and the intersection region contributes $0$. Therefore, $\abs{\calD_1(y \mid (x_1, x_2) ) - \calD_2( y\mid (x_1, x_2) ) } = 2\abs{x_2}/\sigma$ and \eqref{eqn:tv_computation} can be evaluated as 
\begin{equation*}
\begin{split}
    d_{\tv}(\calD_1, \calD_2) &= \frac{1}{2} \int_{\R} 2\indic{ \abs{x_2}  \geq \sigma} +  \frac{2\abs{x_2}}{\sigma} \indic{\abs{x_2} < \sigma } \\
    & =  \prob{\abs{x_2} \geq \sigma} + \frac{1}{\sigma} \expecf{x_2 \sim \calD_1}{\abs{x_2} \indic{\abs{x_2} < \sigma} } \\
    & = 2\epsilon  
\end{split}
\end{equation*}
Finally, we lower bound the parameter distance. Since the coordinates are independent, $\Sigma$ is a diagonal matrix with $\Sigma_{1,1} = \expecf{}{x^2_1} = 1/3$ and $\Sigma_{2,2}= \expecf{}{x^2_2} = \epsilon^{1 - 2/k} + (\epsilon\sigma)^2/3 $. Further, $\Theta_1 - \Theta_2 = (0, 2)$. Thus, $\Norm{\Sigma^{1/2}\Paren{\Theta_1 - \Theta_2}}_2 = 2 \Sigma_{2,2}^{1/2} \geq 2 \epsilon^{1/2 - 1/k} $. For any $\sigma < 1/\epsilon^{1/k}$,
\begin{equation*}
\begin{split}
    \Norm{\Sigma^{1/2}\Paren{\Theta_1 - \Theta_2}}_2   \geq 2~ \epsilon^{1/2 - 1/k} &> 2~\sigma~ \epsilon^{1/2} \\
    & \geq 2~\sqrt{c_k}~\sigma~\epsilon^{1 - 1/k}
\end{split}
\end{equation*}
which concludes the proof. 

\end{proof}



\subsection{Agnostic Model}
Next, consider the setting where we simply observe samples from $(x,y) \sim \calD$, and our goal is to return is to return the minimizer of the squared error, given by $\Theta^* = \expecf{}{x x^{\top}}^{-1}\expecf{}{xy}$. Here, the distribution of the noise is allowed to depend on the covariates arbitrarily. We further assume the noise is hypercontractive and obtain a lower bound proportional to $\epsilon^{1-2/k}$ for recovering an estimator close to $\Theta^*$. This matches the upper boundd obtained in Corollary \ref{cor:dependent_noise}.

\begin{theorem}[Agnostic Model Lower Bound, Theorem \ref{thm:informal_dependent_noise_lowerbound} restated]
\label{thm:lower_bound_dependent_noise}
For any $\epsilon >0$, there exist two distributions $\calD_1, \calD_2$ over $\R^2 \times\R $ such that the marginal distribution over $\R^2$ has covariance $\Sigma$ and is $(c_k,k)$-hypercontractive yet $\Norm{\Sigma^{1/2}(\Theta_{1} - \Theta_{2} ) }_2 =\bigOmega{ \sqrt{c_k}~\sigma~\epsilon^{1 -2/k}} $, where $\Theta_1 , \Theta_2$ be the optimal hyperplanes for $\calD_1 $ and $\calD_2$ respectively, $\sigma = \max( \err_{\calD_1}(\Theta_1),  \err_{\calD_2}(\Theta_2)) < 1/\epsilon^{1/k}$ and the noise is a function of the marginal distribution of $\R^2$.
\end{theorem}
\begin{proof}
We provide a proof for the special case of $k=4$. The same proof extends to general $k$. We again construct a $2$-dimensional instance where the marginal distribution over covariates is identical for $\calD_1$ and $\calD_2$. The pdf is given as follows: for $q\in \{1,2\}$ on the first coordinate,  $x_1$,
\begin{equation*}
\begin{split}
\calD_q(x_1) = \begin{cases}
1/2 , ~~&  \text{if} ~~ x_1 \in [-1,1] \\
0  & \text{otherwise} ~~
\end{cases}
\end{split}
\end{equation*}
and on the second coordinate, $x_2$, 
\begin{equation*}
\begin{split}
\calD_q(x_2) = \begin{cases}
\epsilon/2, ~~&  \text{if} ~~ x_2 \in \{ -1/\epsilon^{1/4}, 1/\epsilon^{1/4} \} \\
\frac{1-\epsilon}{2 }  & \text{if} ~~ x_2 \in [-1 , 1 ]\\
0 & \text{otherwise}
\end{cases}
\end{split}
\end{equation*}
Observe, $\expecf{}{x_1^4} = 1/5$ and $\expecf{}{x_1^2} = 1/3$. Similarly, $\expecf{}{x_2^4} = 1 + (1-\epsilon)/5$ and $\expecf{}{x_2^2} =  \sqrt{\epsilon} +(1-\epsilon)/3$. Therefore, the marginal distribution over $\R^2$ is $(c, 4)$-hypercontractive for a fixed constant $c$. Next, let
\begin{equation}
\label{eqn:dependent_noise}
\begin{split}
\calD_1(y \mid (x_1,x_2) ) &=  x_2  ~~~~~~~\textrm{ and}\\
\calD_2(y \mid (x_1,x_2) ) &= \begin{cases} 0 ~~~~~~~ \text{if} ~~ \abs{x_2} = 1/\epsilon^{1/4} \\
x_2 ~~~~~ \text{otherwise} 
\end{cases}
\end{split}
\end{equation}
Then, 
\begin{equation*}
\begin{split}
    d_{\tv}(\calD_1, \calD_2) & = \frac{1}{2}\int_{\R^2} \calD_1(x_1, x_2) \int_{\R}\abs{\calD_1(y \mid (x_1, x_2) ) - \calD_2( y\mid (x_1, x_2) ) } \\
    & = \frac{1}{2} \int_{\R} \abs{x_2} \indic{ \abs{x_2} = 1/\epsilon^{1/4}}  \\
    & = \epsilon
\end{split}
\end{equation*}
Since the coordinates over $\R^2$ are independent the covariance matrix $\Sigma$ is diagonal, such that $\Sigma_{1,1}  = \expecf{}{x_1^2}  = 1/3$ and $\Sigma_{2,2} = \expecf{}{x_2^2} = \sqrt{\epsilon} +(1-\epsilon)/3$. We can then compute the optimal hyperplanes using normal equations: 
\begin{equation*}
    \Theta_1 = \expecf{x\sim\calD_1}{x x^\top}^{-1} \expecf{x,y\sim \calD_1}{ xy } = \Sigma^{-1}\expecf{x,y\sim \calD_1}{ xy }
\end{equation*}
Observe, using \eqref{eqn:dependent_noise}, 
\begin{equation*}
    \expecf{}{x_1 y} = \int_{\R} x_1 y \calD_1(x_1 y) = \int_{\R} x_1 y \calD_1(x_1) \calD_1(y) = 0
\end{equation*}
since $x_1$ and $y$ are independent.  Further, 
\begin{equation*}
    \expecf{}{x_2 y} = \int_{\R} x_2 y D(x_2, y) = \int_\R x^2_2  D(x_2) = \sqrt{\epsilon} +(1-\epsilon)/3
\end{equation*}
Therefore, $\Theta_1 = (0, 1)$. Similarly, 
\begin{equation*}
    \Theta_2 = \expecf{x\sim\calD_2}{x x^\top}^{-1} \expecf{x,y\sim \calD_2}{ xy } = \Sigma^{-1}\expecf{x,y\sim \calD_2}{ xy }
\end{equation*}
Further, $\expecf{}{x_1 y} =0$. However, 
\begin{equation*}
    \expecf{}{x_2 y} = \int_\R x_2 y \calD_2(x_2, y) = \int_\R  x^2_2 \indic{\abs{x_2} \leq1 }\calD_2(x_2) = 1-\epsilon
\end{equation*}
Therefore, $\Theta_2 = \Paren{0, \frac{1-\epsilon}{1+\sqrt{\epsilon} }}$. Then, 
\begin{equation*}
\begin{split}
    \Norm{\Sigma^{1/2}\Paren{\Theta_1 - \Theta_2} }_2 = \sqrt{\sqrt{\epsilon} +(1-\epsilon)/3} \cdot \frac{\sqrt{\epsilon}+\epsilon}{1+\sqrt{\epsilon}} = \Omega(\sqrt{\epsilon})
\end{split}
\end{equation*}
which concludes the proof. 
\end{proof}

\section{Bounded Covariance Distributions}

In the heavy-tailed setting, the minimal assumption is to consider a distribution over the covariates with bounded covariance. In this setting, we show that robust estimators for linear regression do not exist, even when the underlying linear model has no noise, i.e. the uncorrupted samples are drawn as follows: $y_i = \Iprod{\Theta^*, x_i}$.

\begin{theorem}[Lower Bound for Bounded Covariance Distributions]
\label{thm:lb_bounded_covariance}
For all $\epsilon>0$, there exist two distributions $\calD_1, \calD_2$ over $\R \times \R$ corresponding to the linear model $y = \Iprod{\Theta_1, x}$ and $y = \Iprod{\Theta_2, x}$ respectively, such that the marginal distribution over $\R$ has variance $\sigma$ and $d_{\tv}\Paren{\calD_1, \calD_2} \leq \epsilon$, yet $\abs{\sigma\Paren{\Theta_1 -\Theta_2}} = \Omega(\sigma)$.
\end{theorem}

Our hard instance relies on the so called \textit{Student's t-distribution}, which has heavy tails when the degrees of freedom are close to $2$. 

\begin{definition}[Student's $t$-distribution]
Given $\nu > 1$, Student's $t$-distribution has the following probability density function: 
\begin{equation*}
    f_{\nu}(t) = \frac{ \Gamma\Paren{\frac{\nu + 1}{2} } }{\sqrt{\nu \pi} \hspace{0.05in} \Gamma\Paren{\frac{\nu }{2} } } \Paren{1 + \frac{t^2}{\nu} }^{- \frac{\nu+1}{2}} 
\end{equation*}
where $\Gamma(z) = \int_{0}^\infty x^{z-1} e^x dx $, for $z\in\R$, is the Gamma function. 
\end{definition}

We use the following facts about Student's $t$-distribution: 

\begin{fact}[Mean and Variance]
The mean of Student's $t$-distribution is $\expecf{x\sim f_\nu}{x} = 0$ for $\nu > 1$ and undefined otherwise. The variance of Student's $t$-distribution is 
\begin{equation*}
    \expecf{x\sim f_\nu}{x^2} =
    \begin{cases}
    \infty & \mbox{if }  1 < \nu \leq 2\\
    \frac{\nu}{\nu-2} & \mbox{if }  2 < \nu  \\
    \textrm{undefined } & \textrm{otherwise}
    \end{cases}
\end{equation*}

\end{fact}

The intuition behind our lower bound is to construct a regression instance where the covariates are non-zero only on an $\epsilon$-measure support and are heavy tailed when non-zero. As a consequence, the adversary can introduce a distinct valid regression instance by changing a different $\epsilon$-measure of the support. It is then information-theoretically impossible to distinguish between the true and the planted models.  

\begin{proof}[Proof of Theorem \ref{thm:lb_bounded_covariance}]
We construct a $1$-dimensional instance where the marginal distribution over covariates is identical for $\calD_1$ and $\calD_2$. The pdf is given as follows: for $q\in \{1,2\}$ the marginal distribution on the covariates is given as follows:
\begin{equation*}
\begin{split}
\calD_q(x) = \begin{cases}
1 -\epsilon  , ~~&  \textrm{if} ~~ x =0 \\
\epsilon\cdot f_{2+\epsilon}(x)  & \text{otherwise} ~~
\end{cases}
\end{split}
\end{equation*}
The distribution of the labels is gives as follows: 
\begin{equation*}
    \calD_1\Paren{ y \mid x } = x ~~ \textrm{and} ~~ \calD_2\Paren{ y \mid x } = - x
\end{equation*}
Next, we compute the total variation distance between $\calD_1$ and $\calD_2$. Recall,
\begin{equation}
\begin{split}
    d_\tv(\calD_1, \calD_2) &= \frac{1}{2} \int_{\R \times \R} \left| \calD_1(x,y) - \calD_2(x,y) \right|  \\
    & = \frac{1}{2}\int_{\R} \calD_1(x) \int_{\R}\abs{\calD_1(y \mid x ) - \calD_2( y\mid x ) } \\
    & = \frac{1}{2} \int_{\R}\abs{\calD_1(y \mid x ) - \calD_2( y\mid x ) }\Paren{\indic{x=0} +\indic{x\neq 0}} \\
    & =  \frac{1}{2} \int_{\R}\abs{ 2 x } \indic{x\neq 0} \leq \epsilon
\end{split}
\end{equation}

Observe, since the regression instances have no noise, we can obtain a perfect fit by setting $\Theta_1 = 1$ and $\Theta_2 = -1$. Further, for $q \in\{1,2 \}$,
\begin{equation}
    \expecf{x\sim\calD_q}{x} = \Paren{1-\epsilon}\cdot 0 + \epsilon \cdot \expecf{x\sim f_{2+\epsilon}}{x} = 0 
\end{equation}
and 
\begin{equation}
    \expecf{x\sim\calD_q}{x^2} = \Paren{1-\epsilon}\cdot 0 + \epsilon \cdot \expecf{x\sim f_{2+\epsilon}}{x^2} = \epsilon \cdot \frac{2+\epsilon}{\epsilon} 
\end{equation}
Thus,
\begin{equation}
    \left| \expecf{x\sim\calD_q}{x^2}^{1/2} \Paren{\Theta_1 - \Theta_2} \right| = \Paren{2+\epsilon}^{1/2} \cdot 2 = 2 \sigma
\end{equation}
which completes the proof. We note that the $4$-th moment of $f_{2+\epsilon}(t)$ is infinite and thus it is not hypercontractive, even for $k=4$. 
\end{proof}

\section*{Acknowledgment}
We thank Sivaraman Balakrishnan, Sam Hopkins, Pravesh Kothari, Jerry Li, Pradeep Ravikumar and David Wajc for illuminating discussions related to this project. In particular, we thank Pravesh Kothari for answering several technical questions regarding lower bounds appearing in a previous version of \cite{DBLP:conf/colt/KlivansKM18}.

\clearpage

\bibliography{local,custom2,custom,scholar,dblp,mathreview}
\bibliographystyle{alpha}

\appendix
\newpage
\appendix

\section{Robust Identifiability for Arbitrary Noise}
\label{sec:identifiability_dependent_noise}

\begin{proof}[Proof of Corollary \ref{cor:dependent_noise}]
Consider a maximal coupling of $\calD, \calD'$ over $(x,y) \times (x',y')$, denoted by $\calG$, such that the marginal of $\calG$ $(x,y)$ is $\calD$, the marginal on $(x', y')$ is $\calD'$ and $\mathbb{P}_{\calG}[ \indic{(x,y) = (x',y')}] = 1-\epsilon$. Then, for all $v$, 
\begin{equation}
\label{eqn:dependent_main_bound}
\begin{split}
    \Iprod{v, \Sigma_{\calD} (\Theta_{\calD} - \Theta_{\calD'})  } & = \expecf{\calG}{\Iprod{v, x x^{\top} (\Theta_{\calD} - \Theta_{\calD'}) + xy - xy  }} \\
    & =  \expecf{\calG}{\Iprod{v, x \Paren{ \Iprod{x, \Theta_{\calD}}  - y}  }} + \expecf{\calG}{\Iprod{v, x \Paren{  y - \Iprod{x, \Theta_{\calD'}} }  }} 
\end{split}
\end{equation}

Since $\Theta_\calD$ is the minimizer for the least squares loss, we have the following gradient condition : for all $v \in \mathbb{R}^d$,
\begin{equation}
    \label{eqn:dependent_grad}
    \expecf{(x,y)\sim \calD}{ \Iprod{v,  ( \Iprod{x, \Theta_{\calD}} - y) x } } = 0
\end{equation}
Since $\calG$ is a coupling, using the gradient condition \eqref{eqn:dependent_grad} and using that $1 = \indic{(x,y) = (x',y')}$ $+ \indic{(x,y) \neq (x',y')}$, we can rewrite equation \eqref{eqn:dependent_main_bound} as

\begin{equation}
\label{eqn:dependent_main_bound_restated}
\begin{split}
    \Iprod{v, \Sigma_{\calD} (\Theta_{\calD} - \Theta_{\calD'})  }
    & =   \expecf{\calG}{\Iprod{v, x \Paren{  y - \Iprod{x, \Theta_{\calD'}}}} \indic{(x,y) = (x',y')} } \\
    & \hspace{0.15in} +  \expecf{\calG}{\Iprod{v, x \Paren{  y - \Iprod{x, \Theta_{\calD'}}}} \indic{(x,y) \neq (x',y')} }    \\
    & = \expecf{\calG}{\Iprod{v, x' \Paren{  y' -\Iprod{x', \Theta_{\calD'}}}} \indic{(x,y) = (x',y')} } \\
    & \hspace{0.15in} +  \expecf{\calG}{\Iprod{v, x \Paren{  y - \Iprod{x, \Theta_{\calD'}}}} \indic{(x,y) \neq (x',y')} }    \\
\end{split}
\end{equation}
Consider the first term in the last equality above. Using the gradient condition for $\Theta_{\calD'}$ along with Hölder's Inequality, we have

\begin{equation}
\label{eqn:dependent_hypercontactivity_d'}
\begin{split}
     \Big|\mathbb{E}_{\calG} \Big[ &\Iprod{v, x'  \Paren{  y' -\Iprod{x', \Theta_{\calD'}}}}  \indic{(x,y) = (x',y')}  \Big]\Big| \\
    & =  \Big|\expecf{\calD'}{\Iprod{v, x' \Paren{  y' -\Iprod{x', \Theta_{\calD'}}}}  }
     - \expecf{\calG}{\Iprod{v, x' \Paren{  y' -\Iprod{x', \Theta_{\calD'}}}} \indic{(x,y) \neq (x',y')} } \Big| \\
    & = \left|\expecf{\calG}{\Iprod{v, x' \Paren{  y' -\Iprod{x', \Theta_{\calD'}}}} \indic{(x,y) \neq (x',y')} } \right| \\ 
    & \leq \left|\expecf{\calG}{ \indic{(x,y) \neq (x',y')}^{k/(k-2)} }^{(k-2)/k} \right| \cdot \left|\expecf{\calD'}{ \Iprod{v, x' \Paren{  y' -\Iprod{x', \Theta_{\calD'}}}}^{k/2}  }^{2/k} \right| \\ 
\end{split}
\end{equation}
Observe, since $\calG$ is a maximal coupling $\expecf{\calG}{ \indic{(x,y) \neq (x',y')}}^{(k-2)/k} \leq \epsilon^{1-2/k}$. Here, we no longer have independence of the noise and the covariates, therefore using Cauchy-Schwarz 
\[
\expecf{\calD'}{ \Iprod{v, x'}^{k/2} \cdot  \Paren{  y' -\Iprod{x', \Theta_{\calD'}}}^{k/2}  }  \leq \Paren{\expecf{\calD'}{ \Iprod{v, x'}^{k} } \expecf{\calD'}{  \Paren{  y' -\Iprod{x', \Theta_{\calD'}}}^{k}} }^{1/2}
\]
By hypercontractivity of the covariates and the noise, we have 

\[
\expecf{\calD'}{ \Iprod{v, x'}^{k} }^{1/k} \expecf{\calD'}{  \Paren{  y' -\Iprod{x', \Theta_{\calD'}}}^{k}  }^{1/k} \leq \bigO{ c_k ~ \eta_k } \Paren{ v^{\top}\Sigma_{\calD'}v}^{1/2} \expecf{x',y'\sim\calD'}{\Paren{y' - \Iprod{x', \Theta_{\calD'}} }^2 }^{1/2}
\]
Therefore, we can restate \eqref{eqn:dependent_hypercontactivity_d'} as follows 
\begin{equation}
\label{eqn:dependent_hypercontractive_bound_d'_noise}
\begin{split}
\Big| \expecf{\calG}{\Iprod{v, x'  \Paren{  y' -\Iprod{x', \Theta_{\calD'}}}}  \indic{(x,y) = (x',y')}  }\Big| & \leq \bigO{ c_k ~ \eta_k ~\epsilon^{\frac{k-2}{k}} } \Paren{ v^{\top}\Sigma_{\calD'}v}^{\frac{1}{2}} \\
& \hspace{0.15in} \expecf{x',y'\sim\calD'}{\Paren{y' - \Iprod{x', \Theta_{\calD'}} }^2 }^{\frac{1}{2}}
\end{split}
\end{equation}
It remains to bound the second term in the last equality of equation \eqref{eqn:dependent_main_bound_restated}, and we proceed as follows : 

\begin{equation}
\label{eqn:dependent_hypercontactivity_d}
\begin{split}
    \expecf{\calG}{\Iprod{v, x \Paren{  y - \Iprod{x, \Theta_{\calD'}}  }} \indic{(x,y) \neq (x',y')} } 
    & =  \expecf{\calG}{\Iprod{v, x  x^{\top} \Paren{ \Theta_\calD - \Theta_{\calD'}}} \indic{(x,y) \neq (x',y')} } \\
    & \hspace{0.15in}+ \expecf{\calG}{\Iprod{v, x \Paren{  y - \Iprod{x, \Theta_{\calD}}}} \indic{(x,y) \neq (x',y')} }\\
\end{split}
\end{equation}
We bound the two terms above separately.  Observe, applying Hölder's Inequality to the first term, we have
\begin{equation}
\label{eqn:dependent_tricky_term}
    \begin{split}
         \expecf{\calG}{\Iprod{v, x  x^{\top} \Paren{ \Theta_\calD - \Theta_{\calD'}}} \indic{(x,y) \neq (x',y')} } & \leq  \expecf{\calG}{\indic{(x,y) \neq (x',y')} }^{\frac{k-2}{k}}  \expecf{\calG}{\Iprod{v, x  x^{\top} \Paren{ \Theta_\calD - \Theta_{\calD'}}}^{\frac{k}{2}}}^{\frac{2}{k}}\\
         & \leq  ~\epsilon^{\frac{k-2}{k}} \expecf{\calG}{\Iprod{v, x  x^{\top} \Paren{ \Theta_\calD - \Theta_{\calD'}}}^{\frac{k}{2}}}^{\frac{2}{k}} \\
    \end{split}
\end{equation}
To bound the second term in equation \ref{eqn:dependent_hypercontactivity_d}, we again use Hölder's Inequality followed by Cauchy-Schwarz noise and covariates.

\begin{equation}
\label{eqn:dependent_hypercontract_holder}
    \begin{split}
        \expecf{\calG}{\Iprod{v, x \Paren{  y - \Iprod{x, \Theta_{\calD}}}} \indic{(x,y) \neq (x',y')} } & \leq \expecf{\calG}{ \indic{(x,y) \neq (x',y')} }^{\frac{k-1}{k}}\expecf{\calG}{\Iprod{v, x \Paren{  y - \Iprod{x, \Theta_{\calD}}}}^{k} }^{\frac{1}{k}} \\
        &\leq \epsilon^{\frac{k-2}{k}} \expecf{x\sim \calD}{\Iprod{v,x}^{k/2} }^{2/k} \expecf{x, y\sim \calD}{ \Paren{  y - \Iprod{x, \Theta_{\calD}}}^{k/2} }^{2/k}  \\
        & \leq  \epsilon^{\frac{k-2}{k}}~ c_k ~\eta_k \Paren{ v^{\top} \Sigma_{\calD} v}^{1/2} \expecf{x,y\sim\calD}{ (y -  \Iprod{x, \Theta_{\calD}})^2}^{1/2} 
    \end{split}
\end{equation}
where the last inequality follows from hypercontractivity of the covariates and noise. 
Substituting the upper bounds obtained in Equations \eqref{eqn:dependent_tricky_term} and \eqref{eqn:dependent_hypercontract_holder} back in to \eqref{eqn:dependent_hypercontactivity_d},
\begin{equation*}
\begin{split}
    \expecf{\calG}{\Iprod{v, x \Paren{  y - \Iprod{x, \Theta_{\calD'}}}} \indic{(x,y) \neq (x',y')} } & \leq\epsilon^{\frac{k-2}{k}} \expecf{\calG}{\Iprod{v, x  x^{\top} \Paren{ \Theta_\calD - \Theta_{\calD'}}}^{\frac{k}{2}}}^{\frac{2}{k}} \\
    & \hspace{0.15in} +   \epsilon^{\frac{k-2}{k}}~ c_k ~\eta_k \Paren{ v^{\top} \Sigma_{\calD} v}^{1/2} \expecf{x,y\sim\calD}{ (y -  \Iprod{x, \Theta_{\calD}})^2}^{1/2} 
\end{split}
\end{equation*}
Therefore, we can now upper bound both terms in Equation \eqref{eqn:dependent_main_bound_restated} as follows: 

\begin{equation}
\label{eqn:dependent_main_bound_rerestated}
\begin{split}
    \Iprod{v, \Sigma_{\calD} (\Theta_{\calD} - \Theta_{\calD'})  } & \leq \bigO{ c_k ~ \eta_k ~\epsilon^{\frac{k-2}{k}}} \Paren{ v^{\top}\Sigma_{\calD'}v}^{1/2} \expecf{x',y'\sim\calD'}{\Paren{y' - \Iprod{x', \Theta_{\calD'}} }^2 }^{1/2} \\
    & \hspace{0.15in} + \bigO{\epsilon^{\frac{k-2}{k}}} \expecf{\calG}{\Iprod{v, x  x^{\top} \Paren{ \Theta_\calD - \Theta_{\calD'}}}^{k/2}}^{2/k} \\
    & \hspace{0.15in} + \bigO{\epsilon^{\frac{k-2}{k}}~ c_k ~\eta_k} \Paren{ v^{\top} \Sigma_{\calD} v}^{1/2} \expecf{x,y\sim\calD}{ (y -  \Iprod{x, \Theta_{\calD}})^2}^{1/2} 
\end{split}
\end{equation}
Recall, since the marginals of $\calD$ and $\calD'$ on $\R^d$ are $(c_k, k)$-hypercontractive and $\Norm{\calD - \calD' }_{\tv} \leq \epsilon$, it follows from Fact \ref{fact:hypercontractive_covariance} that 
\begin{equation}
\label{eqn:dependent_cov_closeness_under_tv}
    \Paren{1- 0.1}\Sigma_{\calD'} \preceq \Sigma_{\calD} \preceq \Paren{1+0.1}\Sigma_{\calD'}
\end{equation}
when $\epsilon \leq \bigO{(1/c_k k)^{k/(k-2)}}$.
Now, consider the substitution $v = \Theta_\calD - \Theta_{\calD'}$. Observe, 
\begin{equation}
\label{eqn:dependent_hypercontractivity_along_regressor}
\begin{split}
\expecf{\calG}{\Iprod{v, x  x^{\top} \Paren{ \Theta_\calD - \Theta_{\calD'}}}^{k/2}}^{2/k} & = \expecf{\calD}{\Iprod{ x,   \Paren{ \Theta_\calD - \Theta_{\calD'}}}^{k}}^{2/k} \\
& \leq c^2_k  \Norm{\Sigma^{1/2}_{\calD} (\Theta_{\calD} - \Theta_{\calD'}) }^2_2
\end{split}
\end{equation}
Then, using the bounds in \eqref{eqn:dependent_cov_closeness_under_tv} and \eqref{eqn:dependent_hypercontractivity_along_regressor} along with $v = \Theta_{\calD} - \Theta_{\calD'}$ in 
Equation \ref{eqn:dependent_main_bound_rerestated}, we have
\begin{equation}
\label{eqn:dependent_main_bound_rererestated}
\begin{split}
    \Paren{1- \bigO{\epsilon^{\frac{k-2}{k}} c_k^2 } } \Norm{\Sigma^{1/2}_{\calD} (\Theta_{\calD} - \Theta_{\calD'}) }^2_2 & \leq \bigO{ c_k ~ \eta_k ~\epsilon^{\frac{k-2}{k}}} \Norm{\Sigma^{1/2}_{\calD} (\Theta_{\calD} - \Theta_{\calD'}) }_2  \\
    & \hspace{0.15in} \Paren{ \expecf{x',y'\sim\calD'}{\Paren{y' - \Iprod{x', \Theta_{\calD'}} }^2 }^{\frac{1}{2}} + \expecf{x,y\sim\calD}{ (y -  \Iprod{x, \Theta_{\calD}})^2}^{\frac{1}{2}} } \\
\end{split}
\end{equation}
Dividing out \eqref{eqn:dependent_main_bound_rererestated} by $\Paren{1- \bigO{\epsilon^{\frac{k-2}{k}} c_k^2 } } \Norm{\Sigma^{1/2}_{\calD} (\Theta_{\calD} - \Theta_{\calD'}) }^2_2$ and observing that $\bigO{\epsilon^{\frac{k-2}{k}} c_k^2 }$ is upper bounded by a fixed constant less than $1$ yields the parameter recovery bound.

Given the parameter recovery result above, we bound the least-squares loss between the two hyperplanes on $\calD$ as follows: 
\begin{equation}
    \begin{split}
        \big| \err_{\calD}(\Theta_{\calD}) - \err_{\calD}(\Theta_{\calD'}) \big| & = \Big| \expecf{(x,y)\sim \calD}{\Paren{ y - x^{\top}\Theta_{\calD} }^2   - {\Paren{ y - x^{\top}\Theta_{\calD'} + x^{\top}\Theta_{\calD} - x^{\top}\Theta_{\calD}   }^2 } }  \Big| \\
        & = \Big| \expecf{(x,y)\sim \calD}{ \Iprod{  x,(\Theta_{\calD} - \Theta_{\calD'}) }^2 + 2(y- x^{\top}\Theta_{\calD})x^{\top}(\Theta_{\calD} - \Theta_{\calD'})  } \Big|\\
        & \leq \bigO{ c_k^2~\eta_k^2~\epsilon^{2-4/k} }\Paren{ \expecf{x',y'\sim\calD'}{\Paren{y' - \Iprod{x', \Theta_{\calD'}} }^2 } + \expecf{x,y\sim\calD}{ (y -  \Iprod{x, \Theta_{\calD}})^2} } 
    \end{split}
\end{equation}
where the last inequality follows from observing $\expecf{}{\Iprod{ \Theta_{\calD} - \Theta_{\calD'}, x (y- x^{\top}\Theta_{\calD})} } = 0$ (gradient condition) and squaring the parameter recovery bound. 
\end{proof}

\section{Efficient Estimator for Arbitrary Noise}
\label{sec:arbitrary_noise_polytime}

In this section, we provide a proof of the key SoS lemma required to obtain a polynomial time estimator. The remainder of the proof, including the feasibility of the constraints and rounding is identical to the one presented in Section \ref{sec:efficient_estimator}. 

\begin{lemma}[Robust Identifiability in SoS for Arbitrary Noise]
\label{lem:key_sos_lemma_arbitrary_noise}
Consider the hypothesis of Theorem \ref{thm:optimal_efficient_robust_regression}.
Let $w, x',y'$ and $\Theta$ be feasible solutions for the polynomial constraint system $\calA$. Let $\hat{\Theta} = \arg\min_{\Theta} \frac{1}{n}\sum_{i\in [n]} (y^*_i - \Iprod{x^*_i, \Theta} )^2$ be the empirical loss minimizer on the uncorrupted samples and let $\hat\Sigma = \expecf{}{x_i^* (x_i^*)^{\top}}$ be the covariance of the uncorrupted samples. Then, 
\begin{equation*}
\begin{split}
    \calA\sststile{4k}{w,x',y',\Theta} \Biggl\{  \Norm{ \hat\Sigma^{1/2}\Paren{\hat\Theta - \Theta} }^{2k}_2  &\leq 2^{3k}(2\epsilon)^{k-2} c_k^k ~\eta_k^k ~\sigma^{k/2}  \Norm{ \expecf{}{ x_i' (x_i')^{\top}}^{1/2}\Paren{\hat\Theta - \Theta} }^{k}_2 \\
    & \hspace{0.15in} + 2^{3k} (2\epsilon)^{k-2}c_k^{2k} \Norm{ \hat\Sigma^{1/2} \Paren{\hat\Theta -\Theta}}^{2k}_2 \\
    & \hspace{0.15in}+ 2^{3k} (2\epsilon)^{k-2} c_k^k ~\eta_k^k   \expecf{}{\Paren{ y_i^* - \Iprod{x_i^*,\hat\Theta }}^2 }^{k/2} \Norm{ \hat\Sigma^{1/2}\Paren{\hat\Theta - \Theta} }^{k}_2 \Biggr\}
\end{split}
\end{equation*}
\end{lemma}
\begin{proof}
Consider the empirical covariance of the uncorrupted set given by $\hat\Sigma= \expecf{}{x_i^* (x_i^*)^{\top}}$. Then, using the \ref{eq:sos-substitution}, along with Fact \ref{fact:sos-almost-triangle}
\begin{equation}
\label{eqn:dependent_main_sos_bound}
\begin{split}
    \sststile{2k}{\Theta} \Biggl\{ \Iprod{v, \hat\Sigma \Paren{\hat\Theta - \Theta } }^k & = \Iprod{v, \expecf{}{x^*_i (x_i^*)^{\top}\Paren{\hat\Theta - \Theta } + x_i^*y_i^* - x_i^*y_i^*  }  }^k \\
    &= \Iprod{v, \expecf{}{x^*_i \Paren{ \Iprod{x_i^*,\hat\Theta }  - y_i^* } } + \expecf{}{x^*_i \Paren{ y_i^* - \Iprod{x_i^*,\Theta }  }}  }^k \\
    & \leq 2^k \Iprod{v, \expecf{}{x^*_i \Paren{ \Iprod{x_i^*,\hat\Theta }  - y_i^* } } }^k + 2^k\Iprod{v, \expecf{}{x^*_i \Paren{ y_i^* - \Iprod{x_i^*,\Theta }  }}  }^k  \Biggr\}
\end{split}
\end{equation}
Since $\hat\Theta$ is the minimizer of $\expecf{}{\Paren{ \Iprod{x^*_i, \Theta } - y_i^*  }^2 }$, the gradient condition (appearing in Equation  \eqref{eqn:dependent_grad} of the indentifiability proof) implies this term is $0$. Therefore, 
it suffices to bound the second term. 

For all $i\in [n]$, let $w'_i= w_i$ iff the $i$-th sample is uncorrupted in $\calX_\epsilon$, i.e. $x_i = x_i^*$. Then, it is easy to see that $\sum_i w'_i \geq (1-2\epsilon)n$. Further, since $\calA \sststile{2}{w} \Set{ (1- w_i' w_i)^2 = (1-w_i'w_i)}$, 
\begin{equation}
\label{eqn:dependent_bound_uncorrupted_not_indicated}
    \calA \sststile{2}{w} \Set{\frac{1}{n}\sum_{i \in [n]} (1-w'_iw_i)^2 =  \frac{1}{n}\sum_{i \in [n]} (1-w'_iw_i) \leq 2\epsilon }
\end{equation}
The above equation bounds the uncorrupted points in $\calX_\epsilon$ that are not indicated by $w$. Then, using the \ref{eq:sos-substitution}, along with the SoS Almost Triangle Inequality (Fact \ref{fact:sos-almost-triangle}),

\begin{equation}
\label{eqn:dependent_split_terms}
\begin{split}
    \calA \sststile{2k}{\Theta, w'} \Biggl\{\Iprod{v, \expecf{}{x^*_i \Paren{ y_i^* - \Iprod{x_i^*,\Theta }  }}  }^k & = \Iprod{v, \expecf{}{x^*_i \Paren{ y_i^* - \Iprod{x_i^*,\Theta }(w'_i + 1 -w'_i)  }}  }^k  \\
    & =  \Iprod{v, \expecf{}{w'_ix^*_i \Paren{ y_i^* - \Iprod{x_i^*,\Theta }}} + \expecf{}{(1-w'_i)x^*_i \Paren{ y_i^* - \Iprod{x_i^*,\Theta }}}  }^k \\
    &\leq 2^k\Iprod{v, \expecf{}{w'_ix^*_i \Paren{ y_i^* - \Iprod{x_i^*,\Theta }}}}^k  \\
    & \hspace{0.15in} + 2^k \Iprod{v, \expecf{}{(1-w'_i)x^*_i \Paren{ y_i^* - \Iprod{x_i^*,\Theta }}}}^k \Biggr\}
\end{split}
\end{equation}
Consider the first term of the last inequality in \eqref{eqn:dependent_split_terms}. Observe, since $w'_i x_i^*= w_i w_i'  x'_i$ and similarly, $w'_i y_i^*= w_i w_i'  y'_i$, 
\[\calA \sststile{4}{\Theta, w'} \Set{\expecf{}{w'_ix^*_i \Paren{ y_i^* - \Iprod{x_i^*,\Theta }}} = \expecf{}{w'_i w_i x'_i \Paren{ y_i' - \Iprod{x_i',\Theta }}}} 
\] 
For the sake of brevity, the subsequent statements hold for relevant SoS variables and have degree $O(k)$ proofs.  Using the \ref{eq:sos-substitution}, 

\begin{equation}
\label{eqn:dependent_grad_on_sos_var}
\begin{split}
    \calA \sststile{}{} \Biggl\{ \Iprod{v, \expecf{}{w'_ix^*_i \Paren{ y_i^* - \Iprod{x_i^*,\Theta }}}}^k & = \Iprod{v, \expecf{}{w'_i w_i x'_i \Paren{ y_i' - \Iprod{x_i',\Theta }}}}^k \\
    & =\Iprod{v, \expecf{}{  x'_i \Paren{ y_i' - \Iprod{x_i',\Theta }}} + \expecf{}{(1-w_i'w_i)  x'_i \Paren{ y_i' - \Iprod{x_i',\Theta }}}}^k \\
    & \leq 2^k \Iprod{v, \expecf{}{  x'_i \Paren{ y_i' - \Iprod{x_i',\Theta }}} }^k \\
    & \hspace{0.15in} + 2^k \Iprod{v,  \expecf{}{(1-w_i'w_i)  x'_i \Paren{ y_i' - \Iprod{x_i',\Theta }}} }^k  \Biggr\}
\end{split}
\end{equation}
Observe, the first term in the last inequality above is identically $0$, since we enforce the gradient condition on the SoS variables $x',y'$ and $\Theta$. We can then rewrite the second term using linearity of expectation, followed by applying SoS Hölder's Inequality (Fact \ref{fact:sos-holder}) combined with $\calA \sststile{2}{w} \Set{(1-w'_iw_i)^{2} = 1- w_i'w_i}$ to get 
\begin{equation}
\label{eqn:dependent_independence}
\begin{split}
    \calA \sststile{}{} \Biggl\{ \Iprod{v,  \expecf{}{(1-w_i'w_i) x'_i \Paren{ y_i' - \Iprod{x_i',\Theta }}} }^k & = \expecf{}{\Iprod{v,  (1-w_i') w_i x'_i \Paren{ y_i' - \Iprod{x_i',\Theta }} } }^k \\ 
    &  = \expecf{}{ (1-w_i'w_i) \Iprod{ v,  x'_i}   \Paren{ y_i' - \Iprod{ x_i',\Theta }}  }^k \\
    & \leq \expecf{}{(1-w_i'w_i)}^{k-2}\expecf{}{\Iprod{ v,  x'_i}^{k/2}   \Paren{ y_i' - \Iprod{ x_i',\Theta }}^{k/2} }  \\
    & \leq (2\epsilon)^{k-2} \expecf{}{\Iprod{ v, x'_i}^k }\expecf{}{  \Paren{ y_i' - \Iprod{ x_i',\Theta }}^k } \Biggr\}
\end{split}
\end{equation}
where the last inequality follows from \eqref{eqn:dependent_bound_uncorrupted_not_indicated} and the SoS Cauchy Schwarz Inequality. 
Using the certifiable-hypercontractivity of the covariates, 
\begin{equation}
\label{eqn:dependent_certifiable_hypercontrac_cov}
    \calA \sststile{2k}{w,x'} \Set{\expecf{}{\Iprod{ v,  x'_i}^k} \leq c_k^k \expecf{}{ \Iprod{ v,  x'_i}^2 }^{k/2} = c_k^k  \Iprod{ v,\expecf{}{  x'_i (x'_i)^{\top}}v }^{k/2} }
\end{equation}
Further, using certifiable hypercontractivity of the noise,
\begin{equation}
    \label{eqn:dependent_certibiable_hyperconstrac_noise}
    \calA \sststile{}{} \Set{\expecf{}{\Paren{ y_i' - \Iprod{w_i x_i',\Theta }}^k } \leq \eta_k^k \expecf{}{  (y'_i - \Iprod{x'_i, \Theta})^2 ) }^{k/2} }
\end{equation}
Recall, $\sigma = \expecf{}{  (y'_i - \Iprod{x'_i, \Theta})^2 )}$
Combining the upper bounds obtained in \eqref{eqn:dependent_certifiable_hypercontrac_cov} and \eqref{eqn:dependent_certibiable_hyperconstrac_noise}, and plugging this back into \eqref{eqn:dependent_independence}, we get 
\begin{equation}
\label{eqn:dependent_updated_independence}
    \begin{split}
    \calA \sststile{}{} \Biggl\{ \Iprod{v,  \expecf{}{(1-w_i')  x'_i \Paren{ y_i' - \Iprod{x_i',\Theta }}} }^k \leq (2\epsilon)^{k-2} c_k^k ~\eta_k^k ~\sigma^{k/2}  \Iprod{ v,\expecf{}{  x'_i (x'_i)^{\top}}v }^{k/2}  \Biggr\}
\end{split}
\end{equation}
Recall, we have now bounded the first term of the last inequality in \eqref{eqn:dependent_split_terms}. 
Therefore, it remains to bound the second term of the last inequality in \eqref{eqn:dependent_split_terms}. Using the \ref{eq:sos-substitution}, we have 
\begin{equation}
\label{eqn:dependent_almost_tri_second}
\begin{split}
    \calA \sststile{}{} \Biggl\{\Iprod{v, \expecf{}{(1-w'_i)x^*_i \Paren{ y_i^* - \Iprod{x_i^*,\Theta }}}}^k & = \Iprod{v, \expecf{}{(1-w'_i)x^*_i \Paren{ y_i^* - \Iprod{x_i^*,\Theta -\hat\Theta +\hat\Theta }}}}^k\\
    &\leq 2^k \Iprod{v, \expecf{}{(1-w'_i)x^*_i \Paren{ y_i^* - \Iprod{x_i^*,\hat\Theta }}}}^k \\
    & \hspace{0.15in} + 2^k\Iprod{v, \expecf{}{(1-w'_i)x^*_i \Paren{  \Iprod{x_i^*,\Theta-\hat\Theta }}}}^k \Biggr\} 
\end{split}
\end{equation}
We again handle each term separately. Observe, the first term when decoupled is a statement about the uncorrupted samples. Therefore, using the SoS Hölder's Inequality (Fact \ref{fact:sos-holder}), 
\begin{equation}
\label{eqn:dependent_holder_to_first}
\begin{split}
\calA \sststile{}{} \Biggl\{ \Iprod{v, \expecf{}{(1-w'_i)x^*_i \Paren{ y_i^* - \Iprod{x_i^*,\hat\Theta }}}}^k &= \expecf{}{(1-w'_i)\Iprod{v, x^*_i \Paren{ y_i^* - \Iprod{x_i^*,\hat\Theta }}}}^k \\
& \leq \expecf{}{(1-w'_i)}^{k-2} \expecf{}{\Iprod{v, x^*_i \Paren{ y_i^* - \Iprod{x_i^*,\hat\Theta }}}^{k/2}} \\
& \leq (2\epsilon)^{k-2} \expecf{}{\Iprod{v, x_i^*}^{k}} \expecf{}{\Paren{ y_i^* - \Iprod{x_i^*,\hat\Theta }}^{k} }  \Biggr\}
\end{split}
\end{equation}
Using certifiable hypercontractivity of the $x_i^*$s, 
\[
\expecf{}{\Iprod{v, x_i^*}^k } \leq c_k^k \expecf{}{ \Iprod{v, x_i^*}^2 }^{k/2} = c_k^k \Iprod{v, \hat\Sigma v }^{k/2}
\]
where $\hat\Sigma = \expecf{}{x_i^*( x_i^*)^{\top}}$ and similarly using hypercontractivity of the noise, 
\[
\expecf{}{\Paren{ y_i^* - \Iprod{x_i^*,\hat\Theta }}^k } \leq \eta_k^k \expecf{}{\Paren{ y_i^* - \Iprod{x_i^*,\hat\Theta }}^2 }^{k/2}
\]
Then, by the \ref{eq:sos-substitution}, we can bound \eqref{eqn:dependent_holder_to_first} as follows:

\begin{equation}
\label{eqn:dependent_upper_bound_uncorrupted}
    \calA \sststile{}{} \Biggl\{ \Iprod{v, \expecf{}{(1-w'_i)x^*_i \Paren{ y_i^* - \Iprod{x_i^*,\hat\Theta }}}}^k 
 \leq (2\epsilon)^{k-1} c_k^k ~\eta_k^k   \expecf{}{\Paren{ y_i^* - \Iprod{x_i^*,\hat\Theta }}^2 }^{k/2} \Iprod{v, \hat\Sigma v }^{k/2} \Biggr\}
\end{equation}
In order to bound the second term in \eqref{eqn:dependent_almost_tri_second}, we use the SoS Hölder's Inequality,
\begin{equation}
\label{eqn:dependent_holder_2/k}
\begin{split}
    \calA \sststile{}{} \Biggl\{ \Iprod{v, \expecf{}{(1-w'_i)x^*_i \Paren{  \Iprod{x_i^*,\Theta-\hat\Theta }}}}^k  & =  \expecf{}{ (1-w_i')^{k-2} \Iprod{v, x_i^* \Paren{\Iprod{x_i^*, \Theta-\hat\Theta }} }} \\
    & \leq \expecf{}{1-w'_i}^{k-2} \expecf{}{\Paren{v^{\top}x^*_i (x^*_i)^{\top} (\Theta-\hat\Theta) }^{\frac{k}{2}} }^2 \\
    & \leq (2\epsilon)^{k-2} \expecf{}{\Paren{v^{\top}x^*_i (x^*_i)^{\top} (\Theta-\hat\Theta) }^{\frac{k}{2}} }^2 \Biggr\}
\end{split}
\end{equation}
Combining the bounds obtained in \eqref{eqn:dependent_upper_bound_uncorrupted} and \eqref{eqn:dependent_holder_2/k}, we can restate Equation \eqref{eqn:dependent_almost_tri_second} as follows 
\begin{equation}
\label{eqn:dependent_second_term_in15}
\begin{split}
    \calA \sststile{}{} \Biggl\{\Iprod{v, \expecf{}{(1-w'_i)x^*_i \Paren{ y_i^* - \Iprod{x_i^*,\Theta }}}}^k &\leq 2^k (2\epsilon)^{k-1} c_k^k ~\eta_k^k   \expecf{}{\Paren{ y_i^* - \Iprod{x_i^*,\hat\Theta }}^2 }^{k/2} \Iprod{v, \hat\Sigma v }^{k/2} \\
    & \hspace{0.15in} + 2^k(2\epsilon)^{k-2} \expecf{}{\Paren{v^{\top}x^*_i (x^*_i)^{\top} (\Theta-\hat\Theta) }^{\frac{k}{2}} }^2 \hspace{0.2in} \Biggr\} 
\end{split}
\end{equation}
Combining \eqref{eqn:dependent_second_term_in15} with \eqref{eqn:dependent_updated_independence}, we obtain an upper bound for the last inequality in Equation \eqref{eqn:dependent_split_terms}. Therefore, using the \ref{eq:sos-substitution}, we obtain 
\begin{equation}
\begin{split}
    \calA \sststile{}{} \Biggl\{\Iprod{v, \expecf{}{x^*_i \Paren{ y_i^* - \Iprod{x_i^*,\Theta }  }}  }^k 
    &\leq 2^k(2\epsilon)^{k-1} c_k^k ~\eta_k^k ~\sigma^{k/2}  \Iprod{ v,\expecf{}{  x'_i (x'_i)^{\top}}v }^{k/2} \\
    & \hspace{0.15in} + 2^{2k} (2\epsilon)^{k-2} \expecf{}{\Paren{v^{\top}x^*_i (x^*_i)^{\top} (\Theta-\hat\Theta) }^{\frac{k}{2}} }^2 \\
    & \hspace{0.15in}+ 2^{2k} (2\epsilon)^{k-1} c_k^k ~\eta_k^k   \expecf{}{\Paren{ y_i^* - \Iprod{x_i^*,\hat\Theta }}^2 }^{k/2} \Iprod{v, \hat\Sigma v }^{k/2} \Biggr\}
\end{split}
\end{equation}
The remaining proof is identical to Lemma \ref{lem:key_sos_lemma}.
\end{proof}

\section{Proof of Lemma \ref{lem:hypercontractive_lowner_sampling} }

\begin{lemma}[Löwner Ordering for Hypercontractive Samples (restated)]
Let $\calD$ be a $(c_k, k)$-hypercontractive distribution with covariance $\Sigma$ and and let $\calU$ be the uniform distribution over $n$ samples. Then, with probability $1-\delta$, 
\[
\Norm{ \Sigma^{-1/2}\hat\Sigma\Sigma^{-1/2} - I }_F \leq \frac{C_4d^2}{\sqrt{n}\sqrt{\delta}},
\]
where $\hat\Sigma = \frac{1}{n}\sum_{i\in [n]} x_i x_i^{\top}$. 
\end{lemma}

\begin{proof}
Let $\tilde{x_i} = \Sigma^{-1/2} x_i$ and observe that $\frac{1}{n} \sum_i \tilde{x_i} \tilde{x_i}^T = \Sigma^{-1/2}\hat\Sigma\Sigma^{-1/2}$. Moreover, we know that $\Exp[\tilde{x}\tilde{x}^T] = I$. Let $z_{j,k}$ be the $(j,k)$ entry of $\Sigma^{-1/2}\hat\Sigma\Sigma^{-1/2} - I$ given by,
\[ 
z_{j,k} = \frac{1}{n} \sum_{i\in [n]} \tilde{x_i}(j) \tilde{x_i}(k) - \Exp[\tilde{x}(i) \tilde{x}(k)]
\]
Using Chebyshev's inequality, we get that with probability at least $1 - \delta$,
\[ \abs{z_{jk}} \leq \frac{\Exp[\tilde{x}(j)^2 \tilde{x}(k)^2]}{\sqrt{n} \sqrt{\delta}} \textonTop{(i)}{\leq} \frac{\Exp[\tilde{x}(j)^4] + \Exp[\tilde{x}(k)^4]}{2 \sqrt{n} \sqrt{\delta}}, \]
where $(i)$ follows from AM-GM inequality. To bound $\Exp[\tilde{x}(j)^4]$, we use hypercontractivity.
\[ \Exp[\tilde{x}(j)^4] = \Exp[(v^T x)^4] \leq C_4 \Exp[(v^Tx)^2]^2, \]
where $v = \Sigma^{-1/2} e_j$. Plugging this above, we get that $\Exp[\tilde{x}(j)^4] \leq C_4.$ which in turn implies that with probability at least $1-\delta$,
\[ \abs{z_{jk}} \leq \frac{C_4}{\sqrt{n \delta}}. \]
Taking a union bound over $d^2$ entries of $\Sigma^{-1/2}\hat\Sigma\Sigma^{-1/2} - I$, we get that with probability at least $1-\delta$, 
\[ \Norm{ \Sigma^{-1/2}\hat\Sigma\Sigma^{-1/2} - I }_F \leq \frac{C_4d^2}{\sqrt{n}\sqrt{\delta}}  \]
\end{proof}

\end{document}